\documentclass{article}

\usepackage{microtype}
\usepackage{graphicx}
\usepackage{booktabs} 
\usepackage[english]{babel}
\usepackage[numbers,sort]{natbib}
\usepackage[colorinlistoftodos]{todonotes}
\usepackage{bbm}
\renewcommand{\epsilon}{\varepsilon}

\usepackage{amsmath}

\usepackage{comment}
\usepackage{wrapfig}
\usepackage{amsthm}
\usepackage{amssymb}
\usepackage{amsfonts}
\usepackage{dsfont}
\usepackage{thm-restate}
\usepackage{nicefrac}
\usepackage{footnote}
\usepackage{caption}
\usepackage{subcaption}
\usepackage{tikzsymbols}
\usepackage{tikz}
\usepackage{tikzscale}
\usepackage{enumerate}
\usepackage{adjustbox}
\usepackage{mathtools}
\usepackage{graphics}
\usepackage{caption}
\usepackage{algorithm}
\usepackage{algorithmic}
\usepackage{svg}
\usepackage{bbm}
\usepackage[outline]{contour}

\newcommand{\iffullversion}[2]{%
#1%
}

     \usepackage[preprint,nonatbib]{neurips_2020}

\usepackage[utf8]{inputenc} 
\usepackage[T1]{fontenc}    
\usepackage{hyperref}       
\usepackage{url}            
\usepackage{booktabs}       
\usepackage{amsfonts}       
\usepackage{nicefrac}       
\usepackage{microtype}      

\author{%
  Etienne Bamas\thanks{Equal Contribution.} \\
  EPFL, Lausanne, Switzerland\\
  \texttt{etienne.bamas@epfl.ch}\\
  \And
  Andreas Maggiori$^*$\\
  EPFL, Lausanne, Switzerland\\
  \texttt{andreas.maggiori@epfl.ch}\\
  \And 
  Ola Svensson$^*$\\
  EPFL, Lausanne, Switzerland\\
  \texttt{ola.svensson@epfl.ch}
  }

\usepackage[capitalize]{cleveref}

\DeclareMathOperator{\OPT}{OPT}

\DeclareMathOperator{\F}{\mathcal{F}}

\DeclareMathOperator{\A}{\mathcal{A}}
\DeclareMathOperator{\I}{\mathcal{I}}

\newcommand{\logb}[1]{%
  \log\left(#1\right)%
}

\theoremstyle{plain}
\newtheorem{thm}{Theorem} 
\newtheorem*{thm*}{Theorem}
\newtheorem{cor}[thm]{Corollary}

\newtheorem{lem}[thm]{Lemma}

\newtheorem{claim}[thm]{Claim}

\newtheorem*{cla*}{Claim}

\renewcommand{\geq}{\geqslant}
\renewcommand{\leq}{\leqslant}

\title{The Primal-Dual method for Learning Augmented Algorithms}

\begin{document}
\maketitle
\begin{abstract}
The extension of classical online algorithms when provided with predictions is a new and active research area. In this paper, we extend the primal-dual method for online algorithms in order to incorporate predictions that advise the online algorithm  about the next action to take. We use this framework to obtain novel algorithms for a variety of online covering problems. We compare our algorithms to the cost of the true and predicted offline optimal solutions and show that these algorithms outperform any online algorithm when the prediction is accurate while maintaining good guarantees when the prediction is misleading.
\end{abstract}

\pagenumbering{arabic}

\section{Introduction}
\label{sec:introduction}
\paragraph{} In the classical field of online algorithms the input is presented in an online fashion and the algorithm is required to make irrevocable decisions without knowing the future. The performance is often measured in terms of worst-case guarantees with respect to an optimal offline solution. In this paper, we will consider minimization problems and formally, we will say that an online algorithm \(\mathcal{ALG}\) is \(c\)-\textit{competitive} if on any input \(\mathcal I\), the cost \(c_\mathcal{ALG}(\mathcal I)\) of the solution output by algorithm \(\mathcal{ALG}\) on input \(\mathcal I\) satisfies \(c_\mathcal{ALG}(\I)\leq c\cdot \OPT(\mathcal I)\), where \(\OPT(\mathcal I)\) denotes the cost of the offline optimum. Due to the uncertainty about the future, online algorithms tend to be overly cautious which sometimes causes their performance in real-world situations to be far from what a machine learning (ML) algorithm would have achieved. Indeed in many practical applications future events follow patterns which are easily predictable using ML methods. In \cite{DBLP:conf/icml/LykourisV18} Lykouris and Vassilvitskii formalized a general framework for incorporating (ML) predictions into online algorithms and designed an extension of the marking algorithm to solve the online caching problem when provided with predictions. This work was quickly followed by many other papers studying different learning augmented online problems such as scheduling (\cite{DBLP:conf/soda/LattanziLMV20}), caching (\cite{antoniadis2020online,caching_SODA}), ski rental (\cite{Kodialam_ski_rental, GP19, WLW20_skirental,skirentalwithML}), clustering (\cite{LAclusteringwithNN}) and other problems (\cite{BloomFilters,IndykFrequencyestimation}). The main challenge is to incorporate the prediction without knowing how the prediction was computed and in particular without making any assumption on the quality of the prediction. This setting is natural as in real-world situations, predictions are provided by ML algorithms that rarely come with worst-case guarantees on their accuracy. Thus, the difficulty in designing a learning augmented algorithm is to find a good balance: on the one hand, following blindly the prediction might lead to a very bad solution if the prediction is misleading. On the other hand if the algorithm does not trust the prediction at all, it will simply never benefit from an excellent prediction. The aforementioned results solve this issue by designing smart algorithms which exploit the problem structure to achieve a good trade-off between these two cases. In this paper we take a different perspective. Instead of focusing on a specific problem trying to integrate predictions, we show how to extend a very powerful algorithmic method, the Primal-Dual method, into the design of online learning augmented algorithms. We underline that despite the generality of our extension technique, it produces online learning augmented algorithms in a fairly simple and straightforward manner.

\textbf{The Primal-Dual method.} The Primal-Dual (PD) method is a very powerful algorithmic technique to design online algorithms. It was first introduced by \citet{AAA03} to design an online algorithm for the classical online set cover problem and later extended to many other problems such as weighted caching (\cite{weighted_cachingFOCS}), revenue maximization in ad-auctions, TCP acknowledgement and ski rental \cite{ad_auctionsESA}. We mention the survey of Buchbinder and Naor \cite{BuchbinderBook} for more references about this technique. In a few words, the technique consists in formulating the online problem as a linear program \(P\) complemented by its dual \(D\). Subsequently, the algorithm builds online a feasible fractional solution to both the primal \(P\) and dual \(D\). Every time an update of the primal and dual variables is made, the cost of the primal increases by some amount \(\Delta P\) while the cost of the dual increases by some amount \(\Delta D\). The competitive ratio of the fractional solution is then obtained by upper bounding the ratio \(\frac{\Delta P}{\Delta D}\) and using weak duality. The integral solution is then obtained by an online rounding scheme of the fractional solution.

\textbf{Preliminary notions for Learning Augmented (LA) algorithms.} LA algorithms receive as input a prediction \(\A\), an instance \(\mathcal{I}\) which is revealed online, a robustness parameter \(\lambda\), and output a solution of cost \(c_\mathcal{ALG}(\A, \I, \lambda)\). Intuitively, \(\lambda\) indicates our confidence in the prediction with smaller values reflecting high confidence. We denote by \(S\left(\mathcal{A,\mathcal{I}}\right)\) the cost of the output solution on input \(\mathcal{I}\) if the algorithm follows blindly the prediction \(\mathcal{A}\). We avoid defining explicitly  prediction \(\A\) to easily fit different prediction cases. For instance if the prediction \(\mathcal{A}\) is a predicted solution (without necessarily revealing the predicted instance) then following blindly the solution would simply mean to output the predicted solution \(\mathcal{A}\). For each result presented in this paper, it will be clear what is the prediction \(\mathcal{A}\) and the cost \(S(\A, \I)\). Given this, we restate some useful definitions introduced in \cite{DBLP:conf/icml/LykourisV18, skirentalwithML} in our context. For any \(0 < \lambda\leq 1\), we will say that an LA algorithm is \textit{\(C(\lambda)\)-consistent} and \textit{\(R(\lambda)\)-robust} if the cost of the output solution satisfies:
\begin{equation}
    \label{eqn:intro1}
    c_\mathcal{ALG}(\A, \mathcal I, \lambda) \leq  \min \left\lbrace C(\lambda)\cdot S (\A, \mathcal I), R(\lambda)\cdot  \OPT(\I) \right\rbrace
\end{equation}

If \(\A\) is accurate (\(S(\A, \I) \approx \OPT (\I)\)) and at the same time we trust the prediction, we would like our performance to be close to the optimal offline. Thus, ideally \(C(\lambda)\) should approach \(1\) as \(\lambda\) approaches \(0\). On the same spirit, a value of \(\lambda\) close to \(1\) denotes no trust to the prediction, and in that case, our algorithm should not be much worse than the best pure online algorithm. Therefore, \(R(1)\) should be close to the competitive ratio of the best pure online algorithm. We also mention that in some other papers such as \cite{skirentalwithML}, a \textit{smoothness} criterion on the consistency bound is required. In these papers the setting is slightly different. Indeed, prediction \(\A\) is a predicted instance \(\I^{\textit{pred}}\) and the error is defined to describe how far \(\I^{\textit{pred}}\) is from the real instance \(\I\). With this in mind, an algorithm is said to be \textit{smooth} if the performance degrades smoothly as the error increases. We emphasize that, in the applications considered in this paper, this smoothness property is implicitly included in the value of \(S (\A, \mathcal I)\) which degrades smoothly with the quality of the prediction.
\paragraph{Our contributions.} We show how to extend the Primal-Dual method (when predictions are provided) for solving problems that can be formulated as covering problems. The algorithms designed using this technique receive as input a robustness parameter \(\lambda\) and incorporate a prediction. If the prediction is accurate our algorithms can be arbitrarily close to the optimal offline (beating known lower bounds of the classical online algorithms) while being robust to failures of the predictor. We first apply our Primal-Dual Learning Augmented  (PDLA) technique to the online version of the weighted set cover problem, which constitutes the most canonical example of a covering Linear Program (LP). For that problem we show how we can easily modify the Primal-Dual algorithm to incorporate predictions. Even though in this case, prediction may not seem very natural, this result reveals that we can use PDLA to design learning augmented algorithms for the large class of problems that can be formulated as a covering LP. We then continue by addressing problems in which the prediction model is much more natural. Using the PDLA technique, we first design an algorithm which recovers the results of \citet{skirentalwithML} for the ski rental problem, and we also prove that the consistency-robustness trade-off of that algorithm is optimal. We additionally design a learning augmented algorithm for a generalization of the ski rental, namely the Bahncard problem. Finally, we turn our attention to a problem which arises in network congestion control, the TCP acknowledgement problem. We design an LA algorithm for that problem and conduct experiments which confirm our claims. We note that the analysis of the algorithms designed using PDLA is (arguably) simple and boils down to (1) proving robustness with (essentially) the same proof as in the original Primal-Dual technique and (2) proving consistency using a simple charging argument that, without making use of the dual, relates the cost incurred by our algorithms to the prediction.
In addition to that, using PDLA, the design of online LA algorithms is almost automatic. We emphasize that the preexisting online rounding schemes to obtain an integral solution from a fractional solution still apply to our learning augmented algorithms. Hence in all the paper we focus only on building a fractional solution and provide appropriate references for the rounding scheme.

\section{General PDLA method}
\label{sec:set cover main}
In this section we apply PDLA to solve the online weighted set cover problem when provided with predictions. Set cover is arguably the most canonical example of a covering problem and the framework that we develop readily applies to other covering problems. In particular, we use the framework to give tight or nearly-tight LA algorithms for ski rental, Bahncard, and dynamic TCP acknowledgement, which are all problems that can be formulated as covering LPs.  
\begin{wrapfigure}{r}{0.4\textwidth}
  \begin{center}
    \includegraphics[width=0.4\textwidth]{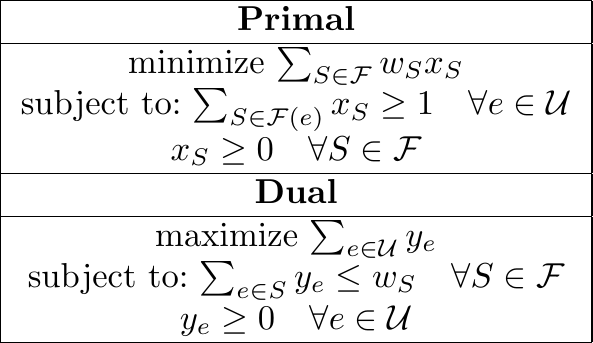}
  \end{center}
  \caption{Primal Dual formulation of weighted set cover}
  \label{fig: primal dual formulation of weighted set cover}
  \vspace{-5pt}
\end{wrapfigure}
\paragraph{The weighted set cover problem.} In this problem, we are given a universe \(\mathcal{U} = \{ e_1, e_2, \dots , e_n \}\) of \(n\) elements and a family \(\mathcal{F}\) of \(m\) sets over this universe, each set \(S\in \mathcal{F}\) has a weight \(w_S\) and each element \(e\) is covered by any set in \(\F (e) = \{ S \in \F \mid e \in S \}\). Let ${d = \max_{e \in \mathcal{U} } |\mathcal{F}(e)|}$ denote the maximum number of sets that cover one element. Our goal is to select sets so as to cover all elements while minimizing the total weight. In its online version, elements are given one by one and it is unknown to the algorithm which elements will arrive and in which order. When a new element arrives, it is required to cover it by adding a new set if necessary. Removing a set from the current solution to decrease its cost is not allowed. Alon et al. in \cite{AAA03} first studied the online version designing an almost optimal \({O(\log n \log d)}\)-competitive algorithm. We note that the \(O(\log n)\) factor comes from the integrality gap of the linear program formulation of the problem (Figure \ref{fig: primal dual formulation of weighted set cover}) while the \(O(\log d)\) is due to the online nature of the problem. Since \citet{AAA03} designed an online rounding scheme at a multiplicative cost of \(O(\log n)\), we will focus on building an increasing fractional solution to the set cover problem (i.e. \(x_S\) can only increase over time for all \(S\)).

\paragraph{PDLA for weighted set cover.} Algorithm \ref{alg:PDLA for Online Weighted Set Cover} takes as input a predicted covering \(\A \subset \F\) and a robustness parameter \(\lambda \in \left[0,1\right]\). While an instance \(\I\) is revealed in an online fashion, an increasing fractional solution \({\{x_S \}_{S \in \F} \in {[0,1]}^{\F}}\) is built. Note that \(\F (e) \cap \A\) are the sets which cover \(e\) in the prediction. To simplify the description, we assume that \({|\F (e) \cap \A | \geq 1}, \forall e\), i.e. the prediction forms a feasible solution. The algorithm without this assumption can be found in \iffullversion{appendix \ref{sec:Set Cover appendix}}{the supplemental material}.

\paragraph{Algorithm Intuition.} We first turn our attention to the original online algorithm of Alon et. al. \cite{AAA03} described in Algorithm \ref{alg:PD for Online Weighted Set Cover}. To get an intuition assume that \(w_S = 1, \forall S\) and consider the very first arrival of an element \(e\). After the first execution of the while loop, \(e\) is covered and \({x_S = \frac{1}{|\F (e)|}, \forall S \in \F (e)}\). In other words, the online algorithm creates a uniform distribution over the sets in \(\F (e)\), reflecting in such a way his unawareness about the future. On the contrary Algorithm \ref{alg:PDLA for Online Weighted Set Cover} uses the prediction to adjust the increase rate of primal variables, augmenting more aggressively primal variables of sets which are predicted to be in the optimal offline solution. Indeed, after the first execution of the while loop, sets which belong to \(\A\) get a value of \({\frac{\lambda}{|\F (e)|} + \frac{1-\lambda}{|{\F} (e) \cap \A|} }\) while sets which are not chosen by the prediction get \(\frac{\lambda}{|\F (e)|}\).

\begin{figure}[!h]
    \begin{minipage}[t]{0.46\textwidth}
       \vspace{0pt}
    \begin{algorithm}[H]
   \caption{\textsc{Primal Dual Method for Online Weighted Set Cover \cite{AAA03}.}}
   \label{alg:PD for Online Weighted Set Cover}
        \begin{algorithmic}
           \STATE {\bfseries Initialize:} \(x_S \leftarrow 0\), \(y_e \leftarrow 0\)  \(\forall S, e\)
           \FORALL{element \(e\) that just arrived}
                \WHILE{\(\sum_{S\in \F (e)} x_S < 1\)}
                    \STATE {\scriptsize \textit{/* Primal Update}}
                    \FORALL{\(S\in \F (e)\)}
                        \STATE \(x_S\leftarrow x_S \left(1+\frac{1}{w_S}\right)+\frac{1}{w_S  |\F (e)|}\)
                    \ENDFOR
                    \STATE {\scriptsize \textit{/* Dual Update}}
                    \STATE \(y_e \leftarrow y_e+1\)
                \ENDWHILE
           \ENDFOR
        \end{algorithmic}
    \end{algorithm}
    \end{minipage}
    \hfill
    \begin{minipage}[t]{0.01\textwidth}
    \vspace{3cm}
       {\Large \contour{black}{\(\Rightarrow\)}}
    \end{minipage}
    \hfill
    \begin{minipage}[t]{0.5\textwidth}
       \vspace{0pt}\raggedright
       \begin{algorithm}[H]
    \small
   \caption{\textsc{PDLA for Online Weighted Set Cover.}}
   \label{alg:PDLA for Online Weighted Set Cover}
\begin{algorithmic}
   \STATE {\bfseries Input:} \(\lambda\), \(\A\)
   \STATE {\bfseries Initialize:} \(x_S \leftarrow 0\), \(y_e \leftarrow 0\)  \(\forall S, e\)
   \FORALL{element \(e\) that just arrived}
        \WHILE{\(\sum_{S\in \F (e)} x_S < 1\)}
            \STATE {\scriptsize \textit{/* Primal Update}}
            \FORALL{\(S\in \F (e)\) and \(S \in \A\)}
                \STATE \({x_S\leftarrow x_S \! \left(\!1+\! \!\frac{1}{w_S}\right) \!+ \!\frac{\lambda}{w_S  |\F (e)|} \! \!+ \! \!\frac{1-\lambda}{w_S |\F (e) \cap \A|} }\)
            \ENDFOR
            \FORALL{\(S\in \F (e)\) and \(S \not \in \A\)}
                \STATE \(x_S\leftarrow x_S \left(1+\frac{1}{w_S}\right)+\frac{\lambda}{w_S  |\F (e)|}\)
            \ENDFOR
            \STATE {\scriptsize \textit{/* Dual Update}}
            \STATE \(y_e \leftarrow y_e+1\)
        \ENDWHILE
   \ENDFOR
\end{algorithmic}
\end{algorithm}
    \end{minipage}
\end{figure}

We continue by exposing our main conceptual contribution. To that end let \(S(\A, \I)\) denote the cost of the covering solution described by prediction \(\A\) on instance \(\I\).
\begin{thm}
\label{thm:set cover} 
Assuming \(\A\) is a feasible solution, the cost of the fractional solution output by Algorithm \ref{alg:PDLA for Online Weighted Set Cover} satisfies
\begin{equation*}
    \label{thm: set cover theorem in the main part of the paper}
    c_{\tiny\mathcal{PDLA}}(\mathcal{A}, \mathcal I, \lambda) \leq \min \left\lbrace O\left(\frac{1}{1-\lambda}\right) \cdot S(\A,\I) ,  O\left(\log \left( \frac{d}{\lambda}\right)\right) \cdot \OPT (\I)\right\rbrace
\end{equation*}
\end{thm}
\begin{proof}[Proof sketch.] The proof is split in two parts. The first part is to bound the cost of the algorithm by the term \({O\left(\frac{1}{1-\lambda}\right)\cdot  S(\A,\I)}\). As mentioned in the introduction we use a charging argument to do so. After each execution of the while loop we can decompose the primal increase into two parts. \(\Delta P_c\) which denotes the increase due to sets in \(\F (e) \cap \A\) and \(\Delta P_u\) which denotes the increase due to sets in \(\F (e) \setminus \A\), thus for the overall primal increase \(\Delta P \) we have \(\Delta P = \Delta P_c + \Delta P_u\). We continue by upper bounding \(\Delta P_u\) as a function of \(\lambda\) and \(\Delta P_c\), that is \(\Delta P_u \leq O (\frac{1 + \lambda}{1 - \lambda})\Delta P_c \), and deducing that \({\Delta P \leq O (\frac{1}{1 -\lambda}) \Delta P_c }\). Now the consistency proof terminates by noting that since \(\Delta P_c\) is generated by sets in the prediction, we can charge this increase to \(S (\A, \I)\). The robustness bound, which is independent of the prediction, is retrieved by mimicking the proof of the original online algorithm of \citet{AAA03}. See \iffullversion{appendix \ref{sec:Set Cover appendix}}{the supplemental material} for more details.
\end{proof}


\section{The Ski rental problem}
\label{sec:alternative ski rental}
\begin{wrapfigure}{r}{0.4\textwidth}
  \begin{center}
    \includegraphics[width=0.4\textwidth]{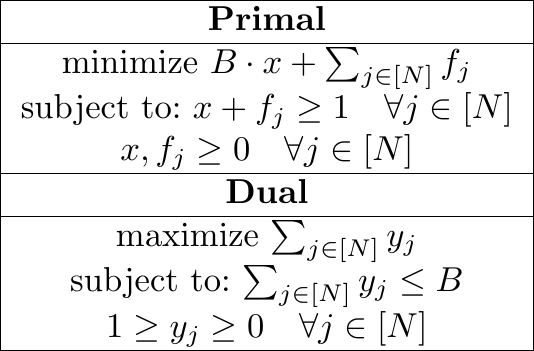}
  \end{center}
  \caption{Primal dual formulation of the ski rental problem.}
  \label{fig:PD for Ski-Rental}
\end{wrapfigure}
As another application of PDLA we design a learning augmented algorithm for one of the simplest and well studied online problems, the ski rental problem. In this problem, every new day, one has to decide whether to rent skis for this day, which costs $1$ dollar or to buy skis for the rest of the vacation at a cost of $B$ dollars. In its offline version the total number of vacation days, $N$, is known in advance and the problem becomes trivial. From the primal-dual formulation of the problem (Figure \ref{fig:PD for Ski-Rental}) it is clear that if $B < N$, the optimal strategy is to buy the skis at day one  while if $B \geq  N$ the optimal strategy is to always rent. In the online setting the difficulty relies in the fact that we do not know $N$ in advance. A deterministic $2$-competitive online algorithm has been known for a long time \cite{KMR86} and a randomized $\frac{e}{e-1}\approx 1.58$-competitive algorithm was also designed later \cite{KMMO90}. Both competitive ratios are known to be optimal for deterministic and randomized algorithms respectively. This problem was already studied in various learning augmented settings \cite{skirentalwithML, GP19, Kodialam_ski_rental, WLW20_skirental}. Our approach recovers, using the primal-dual method, the results of \cite{skirentalwithML}. As in \cite{skirentalwithML} our prediction $\A$ will be the total number of vacation days $N^{\textit{pred}}$.
\newcommand{\Npred}{N^{\textit{pred}}}
\paragraph{PDLA for ski rental.} To simplify the description, we denote an instance of the problem as ${\I = \left(N, B \right)}$ and define the function ${e(z)= (1 + 1/B)^{z\cdot B}}$. Note that if $B\rightarrow \infty$, then $e(z)$ approaches $e^z$ hence the choice of notation. In an integral solution, the variable $x$ is $1$ to indicate that the skis are bought and $0$ otherwise. In the same spirit $f_j$ indicates whether we rent on day $j$ or not. \citet{ad_auctionsESA} showed how to easily turn a fractional monotone solution (i.e. it is not permitted to decrease a variable) to an online randomized algorithm of expected cost equal to the cost of the fractional solution. Hence we focus only on building online a fractional solution. Algorithm \ref{alg:PD for Ski-Rental} is due to \cite{ad_auctionsESA} and uses the Primal-Dual method to solve the problem. Each new day $j$ a new constraint $x+f_j \geq 1$ is revealed. To satisfy this constraint, the algorithm updates the primal and dual variables while trying to maintain (1) the ratio $\Delta P / \Delta D$ as small as possible and (2) the primal and dual solutions feasible. As in the online weighted set cover problem, the key idea for extending Algorithm \ref{alg:PD for Ski-Rental} to the learning augmented Algorithm \ref{alg:PDLA for Ski-Rental} is to use the prediction $N^{\textit{pred}}$ in order to adjust the rate at which each variable is increased. Thus, when $N^{\textit{pred}} > B $ we increase the buying variable more aggressively than the pure online algorithm. Here, the cost of following blindly the prediction $\Npred$ is  $S (\Npred , \I) = B \cdot \mathbbm{1} \{ \Npred > B \} + N \cdot \mathbbm{1} \{ \Npred \leq B \} $.

\begin{minipage}{0.46\textwidth}
\begin{algorithm}[H]
 \caption{\textsc{Primal Dual for Ski-Rental \cite{ad_auctionsESA}.}}
 \label{alg:PD for Ski-Rental}
\begin{algorithmic}
   \STATE {\bfseries Initialize:} $x \leftarrow 0$, $f_j \leftarrow 0 $, $\forall j$
   \STATE $c \leftarrow e(1)$,  $c' \leftarrow 1 $
   \FOR{ each new day $j$ s.t. $ x + f_j < 1$ }
        \STATE \textit{\scriptsize /* Primal Update}
        \STATE $f_j \leftarrow 1 - x $
        \STATE $x \leftarrow (1 + \frac{1}{B}) x + \frac{1}{(c-1) \cdot B}$
        \STATE \textit{\scriptsize /* Dual Update}
        \STATE $y_j \leftarrow c'$
   \ENDFOR
\end{algorithmic}
\end{algorithm}
\end{minipage}
\hfill
\begin{minipage}{0.04\textwidth}
{\Large \contour{black}{$\Rightarrow$}}
\end{minipage}
\hfill
\begin{minipage}{0.46\textwidth}
\begin{algorithm}[H]
 \caption{\textsc{PDLA for Ski-Rental.}}
   \label{alg:PDLA for Ski-Rental}
\begin{algorithmic}
   \STATE {\bfseries Input:} $\lambda$, $N^{\textit{pred}}$
   \STATE {\bfseries Initialize:} $x \leftarrow 0$, $f_j \leftarrow 0 $, $\forall j$
    \IF{ $\Npred \geq  B $}
    \STATE {\scriptsize \textit{/* Prediction suggests buying}}
    \STATE $c \leftarrow e(\lambda)$,  $c' \leftarrow 1 $
    \ELSE
    \STATE \textit{\scriptsize /* Prediction suggests renting}
    \STATE $c \leftarrow e(1/\lambda)$, $c' \leftarrow \lambda $
    \ENDIF
   \FOR{ each new day $j$ s.t. $ x + f_j < 1$ }
        \STATE \textit{\scriptsize /* Primal Update}
        \STATE $f_j \leftarrow 1 - x $
        \STATE $x \leftarrow (1 + \frac{1}{B}) x + \frac{1}{(c-1) \cdot B}$
        \STATE \textit{\scriptsize /* Dual Update}
        \STATE $y_j \leftarrow c'$
   \ENDFOR
\end{algorithmic}
\end{algorithm}
\end{minipage}
\vspace{0.1in}

In the following we assume that either $\lambda B$ or $ B /\lambda$ is an integer (depending on whether $c$ equals $ e(\lambda) $ or $ e(1/\lambda) $ respectively in Algorithm \ref{alg:PDLA for Ski-Rental}). Our results do not change qualitatively by rounding up to the closest integer. See \iffullversion{appendix \ref{sec: Bahncard appendix}}{the supplemental material} for details.

\begin{thm}[PDLA for ski rental]
\label{thm:PDLA ski rental main}
 For any $\lambda \in (0,1]$, the cost of PDLA for ski rental is bounded as follows 
 $$\textit{c}_{\tiny{\mathcal{PDLA}}}(\Npred, \I,\lambda) \leq \min \left\lbrace \frac{\lambda}{1-e(-\lambda)}\cdot S(\Npred, \I), \frac{1}{1-e(-\lambda)}\cdot \OPT  (\I) \right\rbrace$$
\end{thm}
\begin{proof}[Proof sketch.]
The robustness bound is proved essentially using the same proof as for the original analysis of Algorithm \ref{alg:PD for Ski-Rental} in \cite{ad_auctionsESA}. For the consistency bound we first note that after an update the primal increase is $1 + \frac{1}{c-1}$, now depending on the value of $c$ we distinguish between two cases. If $\Npred \geq B$ then Algorithm \ref{alg:PDLA for Ski-Rental} is always aggressive in buying. In this case it is easy to show that at most $\lambda B$ updates are made before we get $x\geq 1$. Once $x\geq 1$, no more updates are needed. Since each aggressive update costs at most $1+\frac{1}{e(\lambda)-1}=\frac{e(\lambda)}{e(\lambda)-1}=\frac{1}{1-e(-\lambda)}$ we get that the total cost paid by Algorithm \ref{alg:PDLA for Ski-Rental} is at most $\frac{\lambda B}{1-e(-\lambda)} = S(\Npred, \I) \cdot \frac{\lambda }{1-e(-\lambda)}$. 
Similarly, in the second case $\Npred < B$ and the algorithm increases the buying variable less aggressively. In this case each update costs at most $1+\frac{1}{e(1/\lambda)-1}=\frac{1}{1-e(-1/\lambda)}$ and at most $N$ of these updates are made therefore Algorithm \ref{alg:PDLA for Ski-Rental} pays at most $\frac{N}{1-e(-1/\lambda)} =S(\Npred, \I)\cdot  \frac{1}{1-e(-1/\lambda)}  $. To conclude the consistency proof, note that $\frac{1}{1-e(-1/\lambda)} \leq \frac{\lambda}{1-e(-\lambda)}$ (\iffullversion{see Lemma \ref{lem:inequalities} inequality \eqref{eqn:lemma12_1b}}{see the supplemental material}).
\end{proof}

In addition to recovering the positive results of \cite{skirentalwithML}, we additionally show in \iffullversion{appendix \ref{sec: Optimality appendix}}{the supplemental material} that this consistency-robustness trade-off is optimal. 

\begin{restatable}{lem}{primelemma}
\label{lem:optimality lemma}
Any $\frac{\lambda}{1-e^{-\lambda}}$-consistent learning augmented algorithm for ski rental has robustness $R(\lambda)\geq \frac{1}{1-e^{-\lambda}}$
\end{restatable}
To emphasize how PDLA permits us to tackle more general problems, we apply the same ideas to a generalization of the ski-rental problem, namely, the Bahncard problem \cite{F01}. This problem models a situation where a tourist travels every day multiple trips. Before any new trip, the tourist has two choices, either to buy a ticket for that particular trip at a cost of $1$ or buy a discount card, at a cost of $B$, that allows to buy tickets at a cheaper price of $\beta < 1 $. The discount card remains valid during $T$ days. Note that ski-rental is modeled by taking $\beta = 0$ and $T\xrightarrow{} \infty$. In the learning augmented version of the problem we are given a prediction $\A$ which consists in a collection of times where we are advised to acquire the discount card. We state the main result on this problem and defer the proof to \iffullversion{Appendix \ref{sec: Bahncard appendix}}{the supplemental material}.

 \begin{thm}[PDLA for the Bahncard problem]
 \label{thm:PDLA bahncard}
  For any $\lambda \in (0,1]$, any $\beta \in [0,1]$ and ${\frac{B}{1-\beta} \xrightarrow[]{} \infty }$, we have the following guarantees on any instance $\mathcal I$ and prediction $\A$
  $$ \textit{cost}_{\tiny{\mathcal{PDLA}}}(\mathcal A, \mathcal I,\lambda) \leq \min \left\lbrace \frac{\lambda}{1-\beta +\lambda \beta }\cdot \frac{e^{\lambda} - \beta}{e^{\lambda} - 1}\cdot S(\mathcal A, I), \frac{e^{\lambda}-\beta}{e^{\lambda}-1}\cdot \OPT  (\I) \right\rbrace$$
 \end{thm}

\section{Dynamic TCP acknowledgement}
\label{sec:tcp ack main}
\begin{wrapfigure}{r}{0.5\textwidth}
  \begin{center}
    \includegraphics[width=0.5\textwidth]{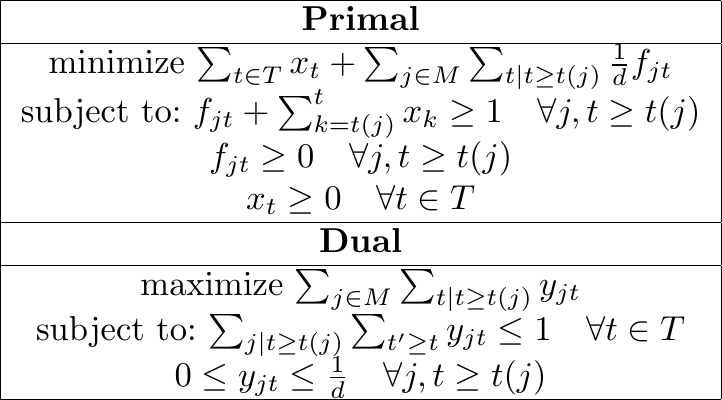}
  \end{center}
  \caption{Primal Dual formulation of the TCP acknowledgement problem}
  \label{fig:PD for TCP}
  \vspace{0.05in}
\end{wrapfigure}
In this section, we continue by applying PDLA to a classic network congestion problem of the Transmission Control Protocol (TCP). During a TCP interaction, a server receives a stream of packets and replies back to the sender acknowledging that each packet arrived correctly. Instead of sending an acknowledgement for each packet separately, the server can choose to delay its response and acknowledge multiple packets simultaneously via a single TCP response. Of course, in this scenario there is an additional cost incurred due to the delayed packets, which is the total latency incurred by those packets. Thus, on one hand sending too many acknowledgments (acks) overloads the network, on the other hand sending one ack for all the packets slows down the TCP interaction. Hence a good trade-off has to be achieved and the objective function which we aim to minimize will be the sum of the total number of acknowledgements plus the total latency. The problem was first modeled by Dooly et al. \cite{DGS98}, where they showed how to solve the offline problem optimally in quadratic time along with a deterministic \(2\)-competitive online algorithm. Karlin et al. \cite{TCPandOtherStories} provided the first \(\frac{e}{e-1}\)-competitive randomized algorithm which was later shown to be optimal by Seiden in \cite{TCP_lowerbound}. The problem was later solved using the primal-dual method by \citet{ad_auctionsESA} who also obtained an \(\frac{e}{e-1}\)-competitive algorithm. Figure \ref{fig:PD for TCP} presents the primal-dual formulation of the problem. In this formulation each packet \(j\) arrives at time \(t(j)\) and is acknowledged by the first ack sent after \(t(j)\). Here, variable \(x_t\) corresponds to sending an ack at time \(t\) and \(f_{jt}\) is set to one (in the integral solution) if packet \(j\) was not acknowledged by time \(t\). The time granularity is controlled by the parameter \(d\) and each additional time unit of latency comes at a cost of \(1/d\). As in the ski rental problem, there is no integrality gap and a fractional monotone solution can be converted to a randomized algorithm in a lossless manner (see \cite{ad_auctionsESA} for more details).
\subsection{The PDLA algorithm and its theoretical analysis}
Our prediction consists in a collection of times \(\A\) in which the prediction suggests sending an ack. Let \(\alpha (t) \) be the next time \(t'\geq t\) when prediction sends an ack. With this definition each packet \(j\), if the prediction is followed blindly, is acknowledged at time \(\alpha (t (j))\) incurring a latency cost of \({\left(\alpha(t(j)) - t(j)\right) \cdot \frac{1}{d}}\). In the same spirit as for the ski rental problem we adapt the pure online Algorithm \ref{alg:Primal-Dual for TCP ack} into the learning augmented Algorithm \ref{alg:LA Primal-Dual for TCP ack}. Algorithm \ref{alg:LA Primal-Dual for TCP ack} adjusts the rate at which we increase the primal and dual variables according to the prediction \(\A\). Thus if a packet \(j\) at time t is "uncovered" (\(\sum_{k = t(j)}^t x_k + f_{jt} < 1\)) by our fractional solution and "covered" by \(\A\) (\(\alpha (t(j)) \leq t\)) we increase \(x_t\) at a faster rate. To simplify the description of Algorithm \ref{alg:LA Primal-Dual for TCP ack} we define \(e(z) = (1+\frac{1}{d})^{z \cdot d}\). To get to the continuous time case, we will take the limit \(d\rightarrow \infty\) so the reader should think intuitively as \(e(z)\approx e^z\). 


\begin{figure}[!h]
    \begin{minipage}[t]{0.46\textwidth}
       \vspace{0pt}
    \begin{algorithm}[H]
  \caption{\textsc{Primal Dual method for TCP acknowledgement \cite{ad_auctionsESA}.}}
  \label{alg:Primal-Dual for TCP ack}
\begin{algorithmic}
  \STATE {\bfseries Initialize:} \(x \leftarrow 0\), \(y\leftarrow 0\)
  \FORALL{times \(t\)}
  \FORALL{packages \(j\) such that \(\sum_{k = t(j)}^t x_k < 1\)}
        \STATE \(c \xleftarrow{} e(1) \), \(c' \xleftarrow{} 1/d\)
        \STATE {\scriptsize \textit{/* Primal Update}}
        \STATE \(f_{jt} \leftarrow 1 - \sum_{k = t(j)}^t x_k\)
        \STATE \(x_t \leftarrow x_t + \frac{1}{d}\cdot \left( \sum_{k = t(j)}^{t} x_k + \frac{1}{c -1}\right)\)
        \STATE {\scriptsize \textit{/* Dual Update}}
        \STATE \(y_{jt} \leftarrow c'\)
  \ENDFOR
  \ENDFOR
\end{algorithmic}
\end{algorithm}
    \end{minipage}
    \hfill
    \begin{minipage}[t]{0.01\textwidth}
    \vspace{3cm}
       {\Large \contour{black}{\(\Rightarrow\)}}
    \end{minipage}
    \hfill
    \begin{minipage}[t]{0.5\textwidth}
       \vspace{0pt}\raggedright
       \begin{algorithm}[H]
\small
  \caption{\textsc{PDLA for TCP acknowledgement}}
  \label{alg:LA Primal-Dual for TCP ack}
\begin{algorithmic}
  \STATE {\bfseries Input: \(\lambda\), \(\A\)}
  \STATE {\bfseries Initialize:} \(x \leftarrow 0\), \(y\leftarrow 0\)
  \FORALL{times \(t\)}
  \FORALL{packages \(j\) such that \(\sum_{k = t(j)}^t x_k < 1\)}
        \IF{\(t \geq \alpha(t(j))\)}
            \STATE {\scriptsize \textit{/* Prediction already acknowledged packet \(j\)}}
            \STATE \(c \xleftarrow{} e(\lambda) \), \(c' \xleftarrow{} 1/d\)
        \ELSE
            \STATE {\scriptsize \textit{/* Prediction did not acknowledge packet \(j\) yet}}
            \STATE \(c \xleftarrow{} e(1/\lambda) \), \(c' \xleftarrow{} \lambda/d\)
        \ENDIF
        \STATE {\scriptsize \textit{/* Primal Update}}
        \STATE \(f_{jt} \leftarrow 1 - \sum_{k = t(j)}^t x_k\)
        \STATE \(x_t \leftarrow x_t + \frac{1}{d}\cdot \left( \sum_{k = t(j)}^{t} x_k + \frac{1}{c -1}\right)\)
        \STATE {\scriptsize \textit{/* Dual Update}}
        \STATE \(y_{jt} \leftarrow c'\)
  \ENDFOR
  \ENDFOR
\end{algorithmic}
\end{algorithm}
    \end{minipage}
\end{figure}

We continue by presenting Algorithm's \ref{alg:LA Primal-Dual for TCP ack} guarantees together with a proof sketch. As before \(\I\) denotes the TCP ack problem instance which is revealed in an online fashion. The full proof is deferred to \iffullversion{appendix \ref{sec:TCP appendix}}{the supplemental material}.

\begin{thm}[PDLA for TCP-ack]
\label{thm:tcp_main}
For any prediction \(\mathcal A\), any instance \(\mathcal I\) of the TCP ack problem, any parameter \(\lambda \in (0,1]\), and \(d\rightarrow \infty\): Algorithm \ref{alg:LA Primal-Dual for TCP ack} outputs a fractional solution of cost at most
\(c_{\tiny{\mathcal{PDLA}}}(\mathcal A, \mathcal I, \lambda) \leq \min \left\lbrace \frac{\lambda}{1-e^{-\lambda }}\cdot S(\mathcal A, \mathcal I),\frac{1}{1-e^{-\lambda}}\cdot \OPT (\I) \right\rbrace \)
\end{thm}
\begin{proof}[Proof sketch]
The two bounds are proven separately. For the robustness bound, while our analysis is slightly more technical, we use the same idea as the original analysis in \cite{ad_auctionsESA}. That is, upper bounding the ratio \(\Delta P /\Delta D\) in every iteration and using weak duality. The consistency proof uses a simple charging scheme that can be seen as a generalization of our consistency proof for the ski rental problem. We essentially have two cases, big (\(c = e(\lambda)\)) and small (\(c = e(1/\lambda)\)) updates. In the case of a small update, a simple calculation reveals that the increase in cost of the solution is at most \(\Delta P=\frac{1}{d}\left(1-\sum_{k=t(j)}^{t}x_k\right)+\frac{1}{d}\left(\sum_{k=t(j)}^{t}x_k + \frac{1}{e(1/\lambda)-1}\right)=\frac{1}{d}\left(1 + \frac{1}{e(1/\lambda)-1}\right)=\frac{1}{d}\cdot \left( \frac{1}{1-e(-1/\lambda)}\right)\). Notice then whenever Algorithm \ref{alg:LA Primal-Dual for TCP ack} does a small update at time \(t\) due to request \(j\), prediction \(\A\) pays a latency cost of \(1/d\) since it has not yet acknowledged request \(j\). Hence the primal increase of cost which is at most \(\frac{1}{d}\cdot \frac{1}{1-e(-1/\lambda)}\) can be charged to the latency cost \(1/d\) paid by \(\A\) with a multiplicative factor \(\frac{1}{1-e(-1/\lambda)} \leq \frac{\lambda}{1-e(-\lambda)}\) (\iffullversion{see Lemma \ref{lem:inequalities}, inequality \eqref{eqn:lemma12_1}}{see the supplemental material}). The case of big updates is slightly different. Consider a time \(t_0\) at which \(\A\) sends an acknowledgement and consider the big updates performed by Algorithm \ref{alg:LA Primal-Dual for TCP ack} for packets \(j\) arrived before that time (\(t(j) \leq t_0\)). We claim that at most \(\left\lceil \lambda d \right\rceil \) such big updates can be made. Indeed, big updates are more aggressive (i.e. \(x_t\) increases faster), and a ``covering'' due to \(\sum_{k = t_0}^t x_k \geq 1\) is reached after only \(\left\lceil \lambda d \right\rceil\) updates (after this point, the packets arrived before time \(t_0\) will never force Algorithm \ref{alg:LA Primal-Dual for TCP ack} to make an update). Thus Algorithm's \ref{alg:LA Primal-Dual for TCP ack} cost due to these big updates is at most \(\left\lceil \lambda d \right\rceil \cdot (\textrm{cost of a big update}) = \left\lceil \lambda d \right\rceil \cdot (\frac{1}{d} \cdot \frac{\lambda}{1 - e(-\lambda)}) \) which can be charged to the cost of 1 incurred by \(\A\) for sending an ack at time \(t_0\).
\end{proof}
\subsection{Experiments}
We present experimental results that confirm the theoretical analysis of Algorithm \ref{alg:LA Primal-Dual for TCP ack} for the TCP acknowledgement problem. The code is publicly available at \url{https://github.com/etienne4/PDLA}. We experiment on various types of distribution for packet arrival inputs. Historically, the distribution of TCP packets was often assumed to follow some Poisson distribution (\cite{TCP_history,Marathe82}). However, it was later shown than this assumption was not always representative of the reality. In particular real-world distributions often exhibit a heavy tail (i.e. there is still a significant probability of seeing a huge amount of packets arriving at some time). To better integrate this in models, heavy tailed distributions such as the Pareto distribution are often suggested (see for instance \cite{mitzenmacher2003,Gong01onthe}). This motivates our choice of distributions for random packet arrival instances. We will experiment on Poisson distribution, Pareto distribution and a custom distribution that we introduce and seems to generate the most challenging instances for our algorithms.
\paragraph{Input distributions.} In all our instances, we set the subdivision parameter $d$ to $100$ which means that every second is split into $100$ time units. Then we define an array of length $1000$ where the $i$-th entry defines how many requests arrive at the $i$-th time step. Each entry in the array is drawn independently from the others from a distribution $\mathcal{D}$. In the case of a Poisson distribution, we set $\mathcal{D}=\mathcal{P}(1)$ (the Poisson distribution of mean $1$). For the Pareto distribution, we choose $\mathcal{D}$ to be the Lomax distribution (which is a special case of Pareto distribution) with shape parameter set to $2$ (\cite{wiki:lomax}). Finally, we define the \textit{iterated} Poisson distribution as follows. Fix an integer $n>0$ and $\mu > 0$. Draw $X_1 \sim \mathcal{P}(\mu)$. Then for $i$ from $2$ to $n$ draw $X_i\sim \mathcal{P}\left(X_{i-1}\right)$. The final value returned is $X_n$. This distribution, while still having an expectation of $\mu$, appears to generate more spikes than the classical Poisson distribution. The interest of this distribution in our case is that it generates more challenging instances than the other two (i.e. the competitive ratios of online algorithms are closer to the worst-case bounds). In our experiments, we choose $\mu = 1$ and $n = 10$. Plots of typical instances under these laws can be seen in \iffullversion{appendix \ref{sec:TCP appendix}}{the supplemental material}. Note that for all these distributions, the expected value for each entry is $1$.
\paragraph{Noisy prediction.} The prediction $\A$ is produced as follows. We perturb the real instances with noise, then compute an optimal solution on this perturbed instance and use this as a prediction. More precisely, we introduce a \textit{replacement} rate $p\in[0,1]$. Then we go through the instance generated according to some distribution $\mathcal{D}$ and for each each entry at index $1 \leq i \leq 1000$, with probability $p$ we set this entry to $0$ (i.e. we delete this entry) and with probability $p$ we add to this entry a random variable $Y\sim \mathcal{D}$. Both operations, adding and deleting, are performed independently of each other. We then test our algorithm with $4$ different values of robustness parameter $\lambda\in \{1, 0.8, 0.6, 0.4\}$.
\paragraph{Results.} The plots in Figure \ref{fig:results} present the average competitive ratios of Algorithm \ref{alg:LA Primal-Dual for TCP ack} over $10$ experiments for each distribution and each value of $\lambda$. As expected, with a perfect prediction, setting a lower $\lambda$ will yield a much better solution while setting $\lambda = 1$ simply means that we run the pure online algorithm of \citet{ad_auctionsESA} (that achieves the best possible competitive ratio for the pure online problem). On the most challenging instances generated by the iterated Poisson distribution (Figure \ref{fig:poisson_it_results}), even with a replacement rate of $1$ where the prediction is simply an instance totally uncorrelated to the real instance, our algorithm maintains good guarantees for small values of $\lambda$. We note that in all the experiments the competitive ratios achieved by Algorithm \ref{alg:LA Primal-Dual for TCP ack} are better than the robustness guarantees of Theorem \ref{thm:tcp_main}, which are $\{1.58,1.68,2.21,3.03\}$ for $\lambda\in \{1, 0.8, 0.6, 0.4\}$ respectively. In addition to that, all the competitive ratios degrade smoothly as the error increases which confirms our earlier discussion about smoothness.
\begin{figure}[H]
\centering
\begin{subfigure}{.32\textwidth}
  \centering
  \includegraphics[width=1.\linewidth]{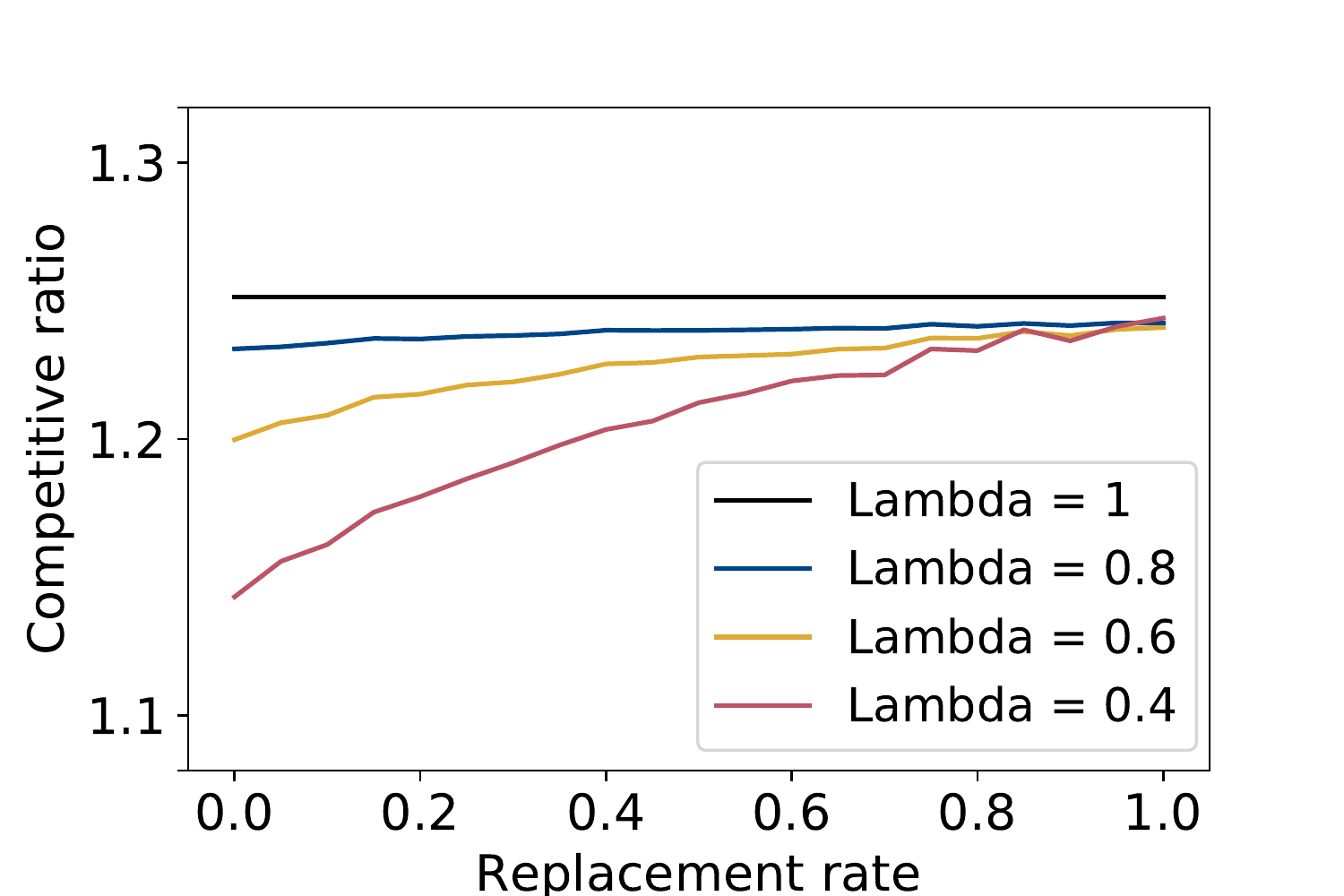}
  \caption{Poisson distribution}
  \label{fig:poisson_results}
\end{subfigure}%
\begin{subfigure}{.32\textwidth}
  \centering
  \includegraphics[width=1.\linewidth]{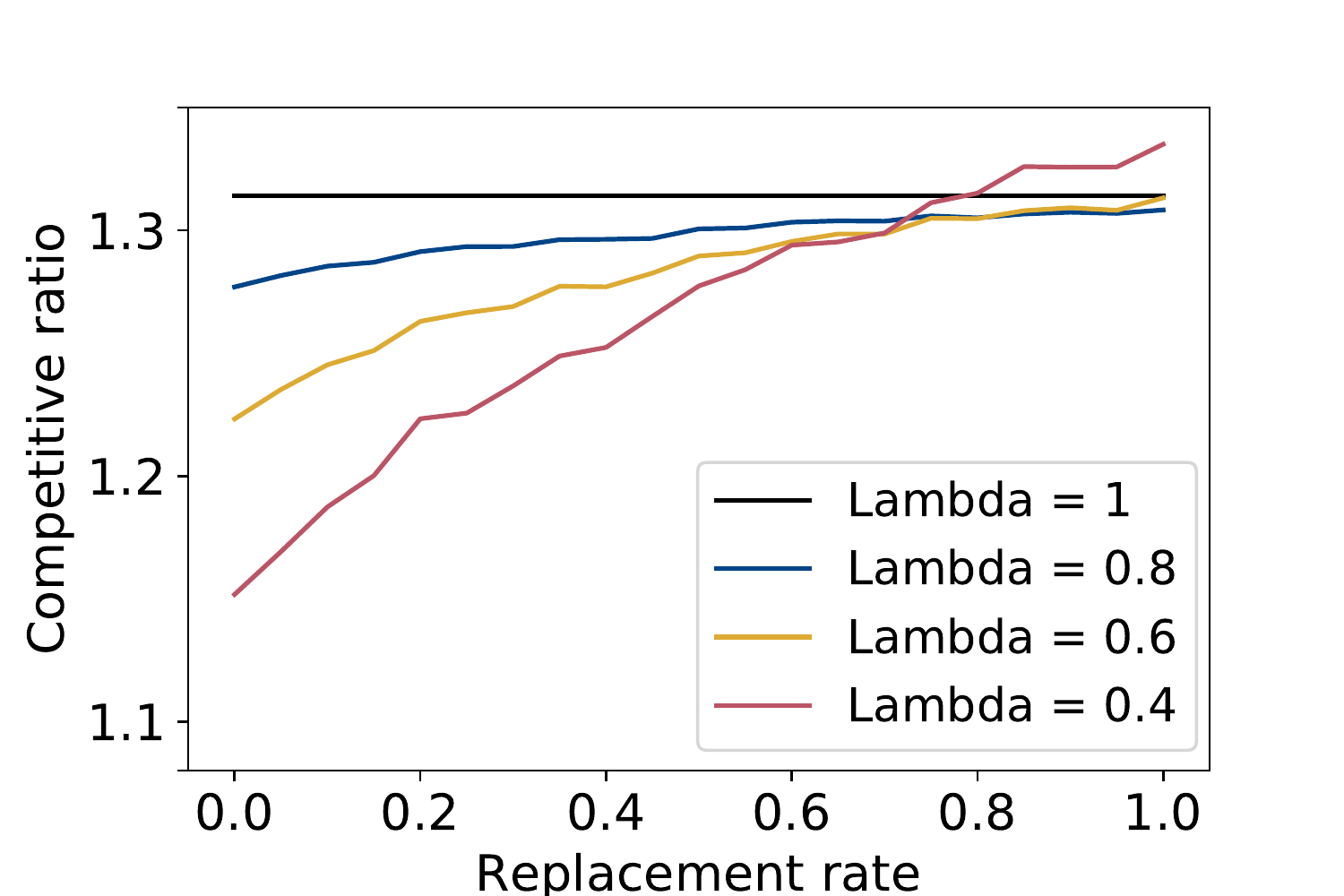}
  \caption{Pareto distribution}
  \label{fig:pareto_results}
\end{subfigure}
\begin{subfigure}{.32\textwidth}
  \centering
  \includegraphics[width=1.\linewidth]{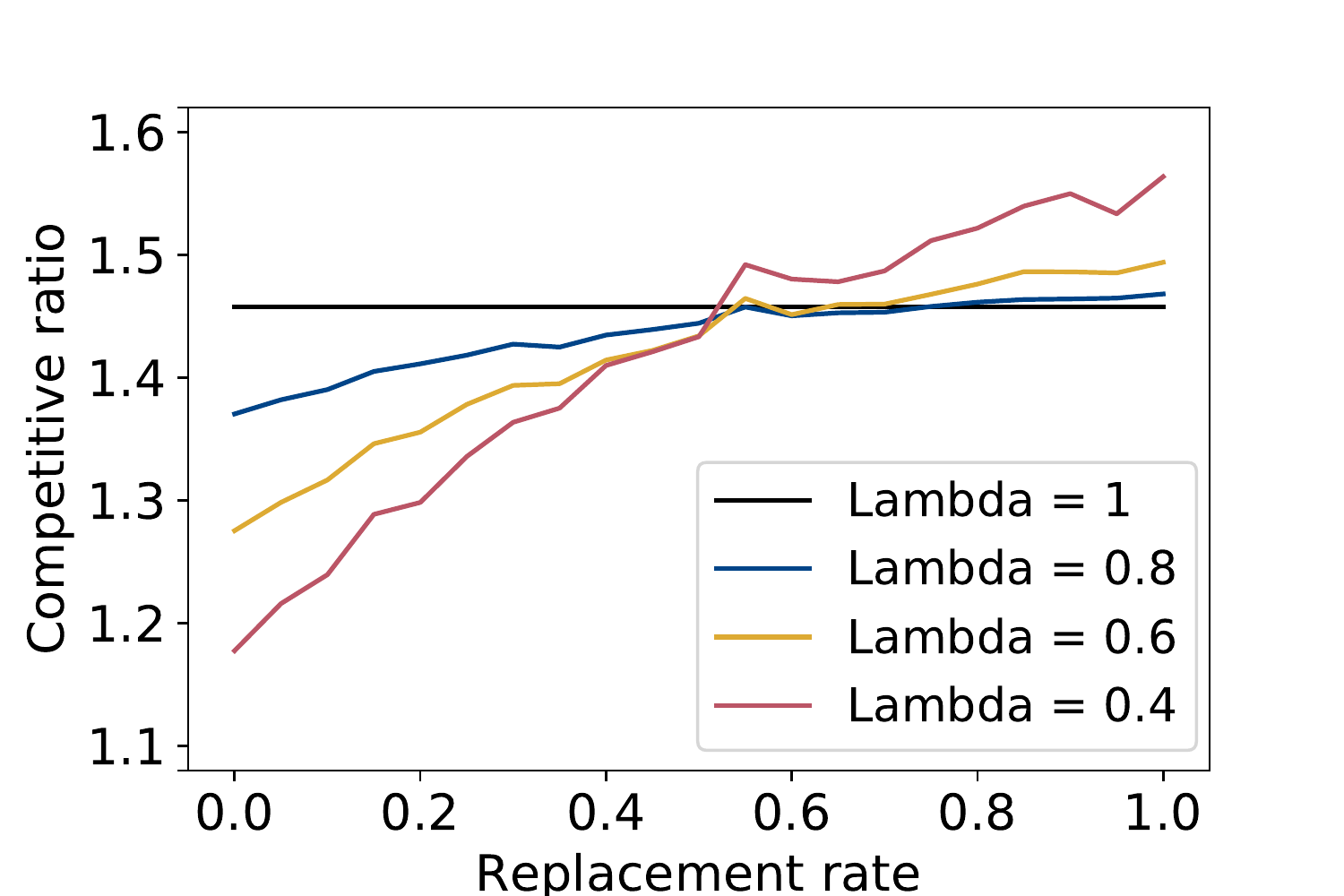}
  \caption{Iterated Poisson distribution}
  \label{fig:poisson_it_results}
\end{subfigure}
\caption{Competitive ratios under various distributions and replacement rates from $0$ to $1$}
\label{fig:results}
\end{figure}

\section{Future Directions}
\label{sec: future directions}
In this paper we present the PDLA technique, a learning augmented version of the classic Primal-Dual technique, and apply it to design algorithms for some classic online problems when a prediction is provided. Since the Primal-Dual technique is used to solve many more covering problems, like for instance weighted caching or load balancing \cite{BuchbinderBook}, an interesting research direction would be to apply PDLA to tackle those problems and (hopefully) get tight consistency-robustness trade-off guarantees (as the one achieved by Algorithm \ref{alg:PDLA for Ski-Rental} and proved in Lemma \ref{lem:optimality lemma}). In addition to that, we suspect that this work might provide insights not only for covering but also for some packing problems which are solved using the Primal-Dual technique in the classic online model (e.g. revenue maximization in ad-auctions \cite{ad_auctionsESA}). Finally, another interesting direction would be to incorporate predictions into the Primal-Dual technique when used to solve covering problems where the objective function is non linear (e.g. convex).

\section*{Broader Impact}
The field of learning augmented algorithms lies in the intersection of machine learning and online algorithms, trying to combine the best of the two worlds. Learning augmented algorithms are particularly suited for critical applications where maintaining worst-case guarantees is mandatory but at the same time predictions about the future are possible. Thus, our work represents a stepping stone towards (easily) integrating ML predictions in such applications, increasing this way the possible benefits of ML to society. PDLA offers a recipe on how to incorporate predictions to tackle classical covering online problems, that is to first solve the online problem using the Primal-Dual technique and then use the prediction to change the rate at which primal and dual variables increase or decrease. We believe that since the idea behind this technique is simple and does not require too much domain-specific knowledge, it might be applicable to different problems and can also be implemented in practice.

\section*{Acknowledgments and Disclosure of Funding}
This research is supported by the Swiss National Science Foundation project 200021-184656 ``Randomness in Problem Instances and Randomized Algorithms''. Andreas Maggiori was supported by the Swiss National Science Fund (SNSF) grant n\textsuperscript{o} $200020\_182517/1$ ``Spatial Coupling of Graphical Models in Communications, Signal Processing, Computer Science and Statistical Physics''.

\nocite{*}
\bibliographystyle{plainnat}
\bibliography{main}


\iffullversion{
\clearpage

\appendix
\onecolumn


\section{Missing proofs for Set Cover}
\label{sec:Set Cover appendix}
We first present the slightly modified algorithm where we do not need the prediction to form a feasible solution. As mentioned in the main paper, when an element $e$ is uncovered by the prediction, i.e. $|\F (e) \cap \A| = 0$, we just run the purely online algorithm ($\lambda = 1$).

\newcommand{\laterm}{\frac{(1-\lambda) \cdot \mathbbm{1} \{ S \in \A\} }{w_S \cdot |\F (e) \cap \A|}}
\begin{algorithm}[H]
   \caption{\textsc{PDLA for Online Weighted Set Cover.}}
   \label{alg:general PDLA for Online Weighted Set Cover}
\begin{algorithmic}
   \STATE {\bfseries Input:} $\lambda$, $\A$
   \STATE {\bfseries Initialize:} $x_S \leftarrow 0$, $y_e \leftarrow 0$  $\forall S, e$
   \FORALL{element $e$ that just arrived}
        \WHILE{$\sum_{S\in \F (e)} x_S < 1$}
            \FORALL{$S\in \F (e)$}
                \IF{ $|\F (e) \cap \A| \geq 1$}
                    \STATE { \small \textit{/* Primal Update (more aggressive if $\mathbbm{1} \{ S \in \A\} = 1$ )}}
                    \STATE ${x_S\leftarrow x_S \left(1+\frac{1}{w_S}\right)+\frac{\lambda}{w_S \cdot |\F (e)|} + \laterm }$
                \ELSE
                    
                    \STATE {\small \textit{/* e is not covered by the prediction}}
                    \STATE $x_S\leftarrow x_S \cdot \left(1+\frac{1}{w_S}\right)+\frac{1}{w_S \cdot |\F (e)|}$
            \ENDIF
            \ENDFOR
                \STATE { \small\textit{/* Dual Update}}
                \STATE $y_e \leftarrow y_e+1$
        \ENDWHILE
   \ENDFOR
\end{algorithmic}
\end{algorithm}

We start by proving that the dual constraints are only violated by a multiplicative factor of $O\left(\log \left(\frac{d}{\lambda}\right)\right)$. Thus, scaling down the dual solution of Algorithm \ref{alg:general PDLA for Online Weighted Set Cover} by $O\left(\log \left(\frac{d}{\lambda}\right)\right)$ creates a feasible dual solution which will permit us to use weak duality.

\begin{lem}
\label{lem:set_constraints}
Let $y$ be the dual solution built by Algorithm \ref{alg:general PDLA for Online Weighted Set Cover}. Then $\frac{y}{\Theta( \log (d/\lambda))}$ is a feasible solution to the dual problem.
\end{lem}
\begin{proof}
The proof essentially follows the same path as in \cite{BuchbinderBook}. The only constraints that can be violated are of the form $\sum_{e\in S} y_e \leq w_S$ for some $S\in \mathcal{F}$. Consider one such constraint. At every update of the primal variable $x_S$ the sum $\sum_{e\in S} y_e$ increases by 1, since the dual variable corresponding to the newly arrived element increases by $1$ . We prove by induction on the number of such updates that at any point in time ${x_S \geq \frac{\lambda}{d} \left(   \left(1 + \frac{1}{w_S}\right)^{\sum_{e \in S} y_e} -1  \right)}$. Indeed, when no update concerning $S$ is done we have that $x_S=0$ and $\sum_{e\in S} y_e = 0$. Suppose this is true after $k$ updates of the variable $x_S$, i.e. $\sum_{e\in S} y_e = k$. Now, assume that a newly arrived element $e^* \in S$ provokes a primal update from  $x_S^{\textit{old}}$ to $x_S^{\textit{new}}$ and increases its dual value by one, i.e. ${y_{e^*}^{\textit{new}} = y_{e^*}^{\textit{old}}+1}$. Then we always have:
\begin{align*}
    x_S^{\textit{new}} \geq &  x_S^{\textit{old}} \cdot \left(1 + \frac{1}{w_S}\right) + \min \left\lbrace \frac{1}{|\F (e)| \cdot w_S} ,  \frac{\lambda}{|\F (e)| \cdot w_S}  + \frac{(1-\lambda)\cdot \mathbbm{1} \{ S \in \A\}}{|\F (e) \cap \A| \cdot w_S}  \right\rbrace \geq \\
     \geq & x_S^{\textit{old}} \cdot \left(1 + \frac{1}{w_S}\right) + \frac{\lambda}{d \cdot w_S}
\end{align*}
Thus, by the induction hypothesis
\begin{align*}
        x_S^{\textit{new}}\geq &   \frac{\lambda}{d} \left(   \left(1 + \frac{1}{w_S}\right)^{\sum_{e \in S \setminus \{ e^*\}} y_e + y_{e^*}^{\textit{old}}} -1  \right) \cdot \left( 1 + \frac{1}{w_S}\right) + \frac{\lambda}{d \cdot w_S} \\
     = & \frac{\lambda}{d} \left(   \left(1 + \frac{1}{w_S}\right)^{\sum_{e \in S \setminus \{ e^*\}} y_e + y_{e^*}^{\textit{new}}} -1  \right) = \frac{\lambda}{d} \left(   \left(1 + \frac{1}{w_S}\right)^{\sum_{e \in S } y_e } -1  \right)
\end{align*}

Moreover, since $w_S \geq 1$, we have that $(1 + 1/w_s )^{w_s} \geq 2$, thus:
\begin{equation*}
x_S \geq \frac{\lambda}{d} \left(   \left(1 + \frac{1}{w_S}\right)^{w_S \cdot \frac{\sum_{e \in S} y_e}{w_S}} -1  \right)\geq \frac{\lambda}{d}\left(2^{\frac{\sum_{e \in S} y_e}{w_S}}-1 \right)
\end{equation*}

We continue by upper bounding the value of $x_S$. Note that once $x_S\geq 1$, no more primal updates can happen, therefore whenever an update is made we have $x_S < 1$ just before the update. Thus:
\begin{align*}
    x_S^{\textit{new}} &\leq  x_S^{\textit{old}} \cdot \left(1 + \frac{1}{w_S}\right) + \max \left\lbrace \frac{\lambda }{w_S \cdot |\F (e)|} + \laterm , \frac{1}{w_S \cdot |\F (e)|} \right\rbrace \\
    &\leq x_S^{\textit{old}} \cdot 2 + 1 \leq 3
\end{align*}
Combining the lower and upper bound on $x_S$ we get that:
\begin{equation*}
    \sum_{e \in S} y_e \leq \logb{\frac{3 d}{\lambda} +1} \cdot w_S = O(\logb{d/\lambda} ) \cdot w_S
\end{equation*}
which concludes the proof.
\end{proof}

\newcommand{\xnew}{x_S^{\textit{new}}}
\newcommand{\xold}{x_S^{\textit{old}}}
\newcommand{\sumoverall}{\sum_{S \in \F (e)}}
\newcommand{\sumovercovered}{\sum_{S \in \F (e) \cap \A}}
\newcommand{\sumoveruncovered}{\sum_{S \in \F (e) \setminus \F (e) \cap \A}}

\begin{lem}[Robustness]
\label{lem: robustness of general PDLA for set cover }
The competitive ratio is always bounded by $O\left(\log\left(\frac{d}{\lambda}\right)\right)$
\end{lem}
\begin{proof}
 We denote as before by $\xold$ and $\xnew$ the primal variables before and after the update respectively. Each time the while loop is executed we have that $\sumoverall \xold < 1$ and  the increase in the dual is $\Delta D = 1$. Denote by $\delta x_S=\xnew-\xold$ the increase of a variable for a specific set $S$. If an element is covered by the prediction then it holds that:
\begin{align*}
    \Delta P  =&   \sumoverall w_S \cdot \delta x_S = \sumovercovered w_S \cdot \delta x_S + \sumoveruncovered w_S \cdot \delta x_S = \\
    = & \sumoverall \left( \xold + \frac{\lambda}{|\F (e)|}\right) + \sumovercovered \frac{(1 - \lambda)}{|\F (e)\cap \mathcal A|} = \sumoverall \xold + \lambda + 1 - \lambda \leq 2
\end{align*}
By repeating the same calculation we get that if an element is uncovered by the prediction then:
\begin{align*}
    \Delta P  =&   \sumoverall w_S \cdot \delta x_S = \sumoverall \left( \xold + \frac{1}{|\F (e)|}\right)  = \sumoverall \xold +  1 \leq 2
\end{align*}

Overall we have that:
\begin{enumerate}
\item At any iteration $\frac{\Delta P }{\Delta D} \leq 2$.
\item The final primal solution is feasible.
\item By Lemma \ref{lem:set_constraints}, denoting $y$ the final dual solution, $\frac{y}{\Theta(\log(d/\lambda)}$ is feasible.
\end{enumerate}
Thus, by weak duality we get that the competitive ratio of Algorithm \ref{alg:general PDLA for Online Weighted Set Cover} is upper bounded by ${2 \cdot O(\logb{d/\lambda} ) }= {O(\logb{d/\lambda}))}$. 
\end{proof}

\newcommand{\lapdcost}{c_{\tiny \mathcal{PDLA}}(\mathcal{A}, \mathcal I, \lambda)}
\newcommand{\nccost}{C_{\textit{nc}}}
\newcommand{\pcost}{S(\A, \I)}
In the following we do not assume that our prediction $\A$ forms a feasible solution. Therefore we will denote by 
\begin{enumerate}
    \item $\pcost$ the cost of the (possibly partial) covering if prediction $\A$ is followed blindly.
    \item $\nccost$ the cost of optimally covering elements which are not covered by the prediction.
    \item $\lapdcost$ the cost of the covering solution calculated by Algorithm \ref{alg:general PDLA for Online Weighted Set Cover}.
\end{enumerate}

\begin{lem}[Consistency]
\label{lem: consistency of general PDLA for set cover}
$\lapdcost \leq O\left(\frac{1}{1-\lambda}\right) \cdot \pcost + O (\logb{d}) \cdot \nccost$
\end{lem}

\begin{proof}
We split the analysis in two parts. First, we look at the case when an element which is uncovered by the prediction arrives. In this case Algorithm \ref{alg:general PDLA for Online Weighted Set Cover} emulates the pure online algorithm ($\lambda = 1$). More precisely, by the same calculations as before, we can show that $y_{nc}$ the solution of the dual problem restricted to the uncovered elements satisfy the property that $\frac{y_{nc}}{O(\log d)}$ is feasible. Therefore for those elements by Lemma \ref{lem: robustness of general PDLA for set cover } the cost of Algorithm \ref{alg:general PDLA for Online Weighted Set Cover} is upper bounded by $O(\log d) \cdot \nccost$. We turn our attention to the more interesting case where the prediction covers an element. In this case, after the execution of the while loop we decompose the primal increase into two parts. $\Delta P_c$ which denotes the increase due to sets $S$ chosen by $\A$ ($\mathbbm{1} \{ S \in \A\} =1$) and $\Delta P_u$ which denotes the increase due to sets $S$ not chosen by the prediction ($\mathbbm{1} \{ S \in \A\} = 0$), thus we have ${\Delta P = \Delta P_c + \Delta P_u}$. Let $c = \{  S \in \F (e) : \mathbbm{1} \{ S \in \A\} =1 \}$ and $u = \{  S \in \F (e) : \mathbbm{1} \{ S \in \A\} =0 \}$. We then have:
\begin{align*}
    \Delta P_c &= \sum_{S \in c} x_S + \frac{\lambda \cdot |c|}{|c| + |u|} + 1-\lambda  \geq \frac{\lambda}{d} + 1 - \lambda\\
    \Delta P_u &= \sum_{S \in u} x_S + \frac{\lambda \cdot |u|}{|c| + |u|}  \leq 1 + \lambda\\
\end{align*}
since , $\frac{|c|}{|c|+ |u|} \geq \frac{1}{d}$ and $\frac{|u|}{|c|+ |u|} \leq 1$.
Combining the two bounds we get that $\Delta P_u \leq \frac{1+ \lambda}{\frac{\lambda}{d} + 1 - \lambda} \cdot \Delta P_c$ and consequently:
\begin{equation*}
    \Delta P \leq \left(1 + \frac{1+ \lambda}{\frac{\lambda}{d} + 1 - \lambda} \right) \Delta P_c = O \left( \frac{1}{1 - \lambda} \right) \Delta P_c
\end{equation*}
Since the cost increase $\Delta P_c$ is caused by sets which are selected by the prediction, we can charge this cost to the corresponding increase of $\pcost$ loosing only a multiplicative $O(1)$ factor. By combining the two cases we conclude the proof.
\end{proof}

\section{Missing proofs for ski rental and the Bahncard problem}
\label{sec: Bahncard appendix}
We detail here the missing proofs from section \ref{sec:alternative ski rental}. We first prove our results regarding the ski rental problem and then focus on the Bahncard problem.

\subsection{The ski rental problem}

We provide here a full proof of Theorem \ref{thm:PDLA ski rental main}. In our setting, the prediction $\A$ is the predicted number of skiing days $\Npred$ and $S(\A,\I)=S(\Npred,\I)=B \cdot \mathbbm{1} \{ \Npred > B \} + N \cdot \mathbbm{1} \{ \Npred \leq B \} $ is the cost of following blindly the prediction. We first prove an easy lemma about the feasibility of the dual solution.
\begin{lem}
\label{lem:dual_ski_rental_feasible}
Let $y$ be the dual solution built by Algorithm \ref{alg:PDLA for Ski-Rental}. Then $y$ is a feasible solution (assuming $\frac{B}{\lambda}$ is integral if the prediction suggests to rent).
\end{lem}
\begin{proof}
To see this, note that the only constraint that might by violated is the constraint $\sum_{j\in [N]}y_j \leq B$. Denote by $S$ the value of the sum $\sum_{j\in [N]}y_j$. Note that once $x\geq 1$, the value of $S$ will never change anymore. The value of $S$ increases by $1$ for every big update and by $\lambda$ for every small update. In the case $\Npred>B$, the algorithm always does big updates (the prediction suggest to buy). We claim that at most $\left\lceil \lambda B\right\rceil$ big updates can be made before $x\geq 1$. We denote $x(k)$ the value of $x$ after $k$ updates. We then prove by induction that $x(k)\geq \frac{e(k/B)-1}{e(\lambda)-1}$ (recall that $e(z) = (1+1/B)^{z \cdot B}\approx e^z$). Clearly, if $k=0$, we have $x(0)\geq 0$. Now assume this is the case for $k$ updates we then have
\begin{align*}
    x(k+1) &= \left(1+\frac{1}{B} \right)\cdot x(k) + \frac{1}{(e(\lambda)-1)\cdot B}\\
    &\geq  \left(1+\frac{1}{B} \right)\cdot \frac{e(k/B)-1}{e(\lambda)-1} + \frac{1}{(e(\lambda)-1)\cdot B}\\
    &= \frac{\left(1+1/B \right)\cdot (e(k/B)-1) + 1/B}{e(\lambda)-1}\\
    &= \frac{e((k+1)/B)-1}{e(\lambda)-1}
\end{align*}
which ends the induction. Hence at most $\left\lceil \lambda B\right\rceil\leq B$ big updates can be made before $x\geq 1$. This implies that $S\leq B$ at the end of the algorithm. In the case where $\Npred \leq B$, we prove in exactly the same way that at most $\left\lceil \frac{B}{\lambda}\right\rceil$ updates are performed before $x\geq 1$. Hence we have that $S\leq \lambda \cdot \left\lceil \frac{B}{\lambda}\right\rceil$. By assumption, we have that $B/\lambda$ is an integer hence $S\leq 1$ and $y$ is again feasible.
\end{proof}

We can finish the main proof.
\begin{proof}[Proof of Theorem \ref{thm:PDLA ski rental main}]
We prove first the robustness bound. By the Lemma \ref{lem:dual_ski_rental_feasible}, we know that the dual solution is feasible. Hence what remains to prove is to upper bound the ratio $\frac{\Delta P}{\Delta D}$ and use weak duality. In the case of a big update we have 
$$ \frac{\Delta P}{\Delta D} = \Delta P = 1+\frac{1}{e(\lambda)-1}=\frac{1}{1-e(-\lambda)}$$
In the case of a small update we have
$$\frac{\Delta P}{\Delta D} = \frac{\Delta P}{\lambda} = \frac{1}{\lambda} \cdot \frac{1}{1-e(-1/\lambda)}\leq \frac{1}{1-e(-\lambda)}$$
where the last inequality comes from Lemma \ref{lem:inequalities} inequality \eqref{eqn:lemma12_1b}. By weak duality, we have the robustness bound.

To prove consistency, we have two cases. If $\Npred\leq B$, then Algorithm \ref{alg:PDLA for Ski-Rental} does at most $N$ updates, each of cost at most $\frac{1}{1-e(-1/\lambda)}$ while the prediction $\A$ pays a cost of $N$. Noting again that, by Lemma \ref{lem:inequalities}, $\frac{1}{1-e(-1/\lambda)}\leq \frac{\lambda}{1-e(-\lambda)}$ ends the proof of consistency in this case. The other case is different. As in the proof of Lemma \ref{lem:dual_ski_rental_feasible}, we still have that $x(k)\geq \frac{e(k/B)-1}{e(\lambda)-1}$ hence at most $\left\lceil \lambda B\right\rceil\leq B$ updates are done by Algorithm, each of cost at most $\frac{1}{1-e(-\lambda)}$ hence a total cost of at most 
$$ \frac{\left\lceil \lambda B\right\rceil}{1-e(-\lambda)}$$ 
Since we assume in this case that $\lambda B$ is integral and that the prediction $\A$ pays a cost of $B$, the competitive ratio is indeed $\frac{\lambda}{1-e(-\lambda)}$
\end{proof}

\subsection{The Bahncard problem}

\paragraph{History of the problem.} The Banhcard problem, which was initially introduced in \cite{F01}, models a situation where a tourist travels every day multiple trips. Before any new trip, the tourist has two choices, either to buy a ticket for that particular trip at a cost of $1$ or buy a discount card, at a cost of $B$, and use this discount card to get a ticket for a price of $\beta < 1 $. The discount card is then valid for rest of that day and for the next $T-1$ days. This generalizes the ski rental problem in several ways, first the discount expires after a fixed amount of time, second buying only offers a discount and not a free trip. Note that if ${\beta = 0}$ and ${T \xrightarrow[]{} \infty} $ we recover the ski-rental problem. Karlin et al.  \cite{TCPandOtherStories} designed an optimal randomized online algorithm of competitive ratio $\frac{e}{e-1+\beta}$ when ${B \xrightarrow[]{} \infty}$.  

\paragraph{PDLA for the Bahncard problem.} We design, using PDLA, a learning augmented algorithm for the Bahncard problem. The final goal is to prove Theorem \ref{thm:PDLA bahncard}. An interesting feature of our algorithm is that, as for the TCP ack problem, it does not need to be given the full prediction in advance. If Bahncards are bought by the prediction $\A$ at a set of times $\{t_1, t_2,\ldots, t_k\}$, the algorithm does not need to know before time $t_i$ that the Bahncard $i$ is bought. For instance we could think of the prediction of an employee of the station giving short-term advice to a traveller every time he shows up at the station.

We now give the primal dual formulation of the Bahncard problem along with its corresponding learning augmented algorithm. We mention that, to the best of our knowledge, no online algorithm using the primal-dual method was designed before. Hence the primal-dual formulation (Figure \ref{fig:PD_Bahncard}) of the problem is new. In an integral solution, we would have $x_t=1$ if the solution buys a Bahncard at time $t$ and $x_t=0$ otherwise. Then $f_j$ represents the fractional amount of trip $j$ done at time $t(j)$ that is bought at full price and $d_j$ the amount of the trip bought at discounted price. The first natural constraint is the one that says that each trip should be paid entirely either in discounted or full price, i.e. $d_j+f_j\geq 1$. We then have the constraint $\sum_{t=t(j)-T}^{t(j)} x_t \geq d_j$ that says that to be able to buy a ticket at discounted price, at least one Bahncard must have been bought in the last $T$ time steps. 
\begin{figure}[h]
    \centering
    \caption{Primal Dual formulation of the Bahncard problem.}
    \label{fig:PD_Bahncard}
    \begin{tabular}{|c|c|}
    \hline
    \textbf{Primal} &  \textbf{Dual} \\
    \hline
     \text{minimize} $B \cdot \sum_{t \in T } x_t + \sum_{j \in M }\beta d_j + f_j$  & \text{maximize} $\sum_{j \in M } c_j$ \\
     \text{subject to:} $\quad d_j+f_j \geq 1 \quad \forall j$ & \text{subject to:} $\quad c_j\leq 1 \quad \forall j$ \\
     $\sum_{t=t(j)-T}^{t(j)} x_t \geq d_j \quad \forall j$ & $c_j-b_j\leq \beta \quad \forall j$ \\
     $x_t\geq 0 \quad \forall t\in \mathcal  T$ & $\sum_{j:t(j)-T\leq  t\leq  t(j)} b_j \leq B \quad \forall t\in \mathcal T$\\
     $d_j,f_j \geq 0 \quad \forall j$ & $c_j,b_j\geq 0 \quad \forall j$\\
     
     \hline 
\end{tabular}{}
    
\end{figure}{}

Following the same idea as for the ski rental problem, we will guide the updates in the primal-dual algorithm with the advice provided. We define a function $e(z)=\left(1+\frac{1-\beta}{B} \right)^{z\cdot (B/(1-\beta))}$. Again for $\frac{B}{1-\beta}\rightarrow \infty$, the reader should think intuitively of $e(z)$ as $e^z$. The parameter $z$ will then take values either $\lambda$ or $1/\lambda$ depending on if we want to do a big or small update in the primal. As for ski rental, when we do a small update, we will need to scale down the dual update by a factor of $\lambda$ to maintain feasibility of the dual solution.

The rule to decide if an update should be big or small is the following: if the prediction $\mathcal A$ bought a Bahncard less than $T$ time steps in the past (i.e. if the predicted solution has currently a valid Bahncard) the update should be big. Otherwise the update should be cautious. In algorithm \ref{alg:LA Primal-Dual for Bahncard}, we denote by $l_\mathcal{A}(t)$ the latest time before time t at which the prediction $\mathcal{A}$ bought a Bahncard. We use the convention that $l_\mathcal{A}(t)=-\infty$ if no Bahncard was bought before time $t$. Of course in this problem it is possible that trips show up while the fractional solution already has a full Bahncard available (i.e. $\sum_{t=t(j)-T}^{t(j)} x_t\geq 1$). In this case there is no point in buying more fractions of a Bahncard and the algorithm will do what we call a \textit{minimal} update. 

\begin{algorithm}[H]
   \caption{\textsc{LA Online Primal-Dual for the Bahncard problem}}
   \label{alg:LA Primal-Dual for Bahncard}
\begin{algorithmic}
   \STATE {\bfseries Input: $\lambda$, $\mathcal A$}
   \STATE {\bfseries Initialize:} $x,d,f \leftarrow 0$, $c,b\leftarrow 0$
   \FORALL{trip $j$}
        \IF{$\sum_{t=t(j)-T}^{t(j)} x_t \geq 1$}
        \STATE $d_j\leftarrow 1$
        \STATE $c_j\leftarrow \beta$
        \ENDIF
        \IF{$\sum_{t=t(j)-T}^{t(j)} x_t < 1$}
        \IF{$t(j)\leq l_\mathcal{A}(t(j))+T$}
        \STATE $d_j \leftarrow \sum_{t=t(j)-T}^{t(j)} x_t$
        \STATE $f_j \leftarrow 1-d_j$
        \STATE $x_{t(j)} \leftarrow x_{t(j)} + \frac{1-\beta}{B}\cdot \left(\sum_{t=t(j)-T}^{t(j)} x_t + \frac{1}{e(\lambda)-1}\right)$
        \STATE $b_j \leftarrow 1-\beta$
        \STATE $c_j\leftarrow b_j+\beta$
        \ENDIF
        \IF{$t(j)> l_\mathcal{A}(t(j))+T$}
        \STATE $d_j \leftarrow \sum_{t=t(j)-T}^{t(j)} x_t$
        \STATE $f_j \leftarrow 1-d_j$
        \STATE $x_{t(j)} \leftarrow x_{t(j)} + \frac{1-\beta}{B}\cdot \left( \sum_{t=t(j)-T}^{t(j)} x_t + \frac{1}{e(1/\lambda)-1}\right)$
        \STATE $b_j \leftarrow \lambda (1-\beta)$
        \STATE $c_j\leftarrow b_j+\beta$
        \ENDIF
        \ENDIF
   \ENDFOR
\end{algorithmic}
\end{algorithm}

We first prove that the dual built by the algorithm is almost feasible.

\begin{lem}
\label{lem:dual_bahncard}
Let $(c,b)$ be the dual solution built by Algorithm \ref{alg:LA Primal-Dual for Bahncard}, then $\frac{(c,b)}{1+(1-\beta)/B}$ is feasible.
\end{lem}
\begin{proof}
Note that the constraints $c_j\leq 1$ and $c_j-b_j\leq \beta$ are clearly maintained by the algorithm. And scaling down both $c$ and $b$ by some factor bigger than $1$ will not alter their feasibility. Hence we focus only on the constraints of the form $\sum_{j:t(j)-T\leq  t\leq  t(j)} b_j \leq B$ for a fixed time $t$. Note that during a minimal update, the value of $b_j$ is not changed hence only small or big updates can alter the value of the sum $\sum_{j:t(j)-T\leq  t\leq  t(j)} b_j$. Similarly as for proofs in ski rental, denote by $b$ the number of big updates that are counted in this sum and by $s$ the number of small updates in this sum. 

We first notice that once we have that $\sum_{t'=t}^{t+T} x_{t'} \geq 1$, no updates that alter the constraint $\sum_{j:t(j)-T\leq  t\leq  t(j)} b_j$ can happen. To see this, note that upon arrival of a trip $j$ between time $t$ and $t+T$, we have $\sum_{t=t(j)-T}^{t(j)} x_t \geq \sum_{t'=t}^{t(j)} x_{t'}=\sum_{t'=t}^{t+T} x_{t'}$.

Denote by $S$ the value of the sum $\sum_{t'=t}^{t+T} x_{t'}$. Note that for a big update, we have that the value of the sum $S$ is increased to at least $S\cdot \left(1+\frac{1-\beta}{B} \right)+\frac{1-\beta}{B}\cdot \frac{1}{e(\lambda)-1}$. Similarly for a small update the new value of the sum is at least $S\cdot \left(1+\frac{1-\beta}{B} \right)+\frac{1-\beta}{B}\cdot \frac{1}{e(1/\lambda)-1}$. Hence we can apply directly Lemma \ref{lem:sequence_lemma} with $d = \frac{B}{1-\beta}$ to conclude that once $b+\lambda s \geq \frac{B}{1-\beta}$, we have that $S\geq 1$. 

Since for a big update, the sum $\sum_{j:t(j)-T\leq  t\leq  t(j)} b_j$ increases by $1-\beta$ and by $\lambda (1-\beta)$ for a small update we can see that the first time the constraint $\sum_{j:t(j)-T\leq  t\leq  t(j)} b_j\leq B$ is violated, we have $S\geq 1$. Now since each update in the sum $\sum_{j:t(j)-T\leq  t\leq  t(j)} b_j$ is of value at most $1-\beta$ we can conclude that at the end of the algorithm, we have $\sum_{j:t(j)-T\leq  t\leq  t(j)} b_j\leq B+1-\beta$ hence the conclusion. 
\end{proof}

We then prove robustness of Algorithm \ref{alg:LA Primal-Dual for Bahncard} by the following lemma. 

\begin{lem}[Robustness]
\label{lem:robustness Bahncard}
For any $\lambda\in (0,1]$ and any $\beta \in [0,1]$, PDLA for the Bahncard problem is $\frac{\left(e(\lambda)-\beta\right)\cdot \left(1+(1-\beta)/B \right)}{e(\lambda)-1}$-robust.
\end{lem}
\begin{proof}
Algorithm \ref{alg:LA Primal-Dual for Bahncard} makes 3 possible types of updates. For a minimal update, we have $\Delta P = \Delta D = \beta$. For a small update we have 
\begin{align*}
    \Delta P &= \left(1-\beta \right)\cdot \left(\sum_{t=t(j)-T}^{t(j)} x_t + \frac{1}{e(1/\lambda)-1}\right) + \beta \cdot \sum_{t=t(j)-T}^{t(j)} x_t + 1-\sum_{t=t(j)-T}^{t(j)} x_t \\
    &= 1+\frac{1-\beta}{e(1/\lambda)-1} = \frac{e(1/\lambda)-\beta}{e(1/\lambda)-1}
\end{align*}{}
and 
$$\Delta D = \lambda (1-\beta) + \beta = \beta (1-\lambda) + \lambda $$
hence the ratio is 
$$\frac{\Delta P}{\Delta D} = \frac{1}{\beta (1-\lambda) + \lambda} \cdot \frac{e(1/\lambda)-\beta}{e(1/\lambda)-1}$$
Similarly in the case of a big update we have 
$$\Delta P =  \frac{e(\lambda)-\beta}{e(\lambda)-1}$$ and $\Delta D  = 1$ which gives a ratio of 
$$\frac{\Delta P}{\Delta D} = \frac{e(\lambda)-\beta}{e(\lambda)-1}$$
We can conclude by Lemma \ref{lem:inequalities} (inequality \eqref{eqn:lemma12_4}) that the ratio of primal cost increase vs dual cost increase is always bounded by 
$$\frac{\Delta P}{\Delta D}\leq \frac{e(\lambda)-\beta}{e(\lambda)-1}$$ Using Lemma \ref{lem:dual_bahncard} along with weak duality is enough to conclude that the cost of the fractional solution built by the algorithm is bounded as follows
$$\textit{cost}_{\tiny{\mathcal{PDLA}}}(\mathcal A, \mathcal I,\lambda)\leq \frac{\left(e(\lambda)-\beta\right)\cdot \left(1+(1-\beta)/B \right)}{e(\lambda)-1}\cdot \OPT$$
which ends the proof.
\end{proof}

For consistency, we analyze the algorithm's cost in two parts. When the heuristic algorithm $\mathcal{A}$ buys its $i$th Bahncard at some time $t_i$, define the interval $I_i=[t_i,t_i+T]$ which represents the set of times during which this specific Bahncard is valid. This creates a family of intervals $I_1,\ldots I_k$ if $\mathcal{A}$ buys $k$ Bahncards. Note that  we can assume that all these intervals are disjoint since if the prediction $\mathcal{A}$ suggests to buy a new Bahncard before the previous one expires, it is always better to postpone this buy to the end of the validity of the current Bahncard.
\begin{lem}
\label{lem:consistency_1_Bahncard}
Denote by $\left(\Delta P\right)_{I_i}$ the increase in the primal cost of Algorithm \ref{alg:LA Primal-Dual for Bahncard} during interval $I_i$ and by cost$(\mathcal{A})_{I_i}$ what prediction $\mathcal{A}$ pays during this same interval $I_i$ (including the buy of the Bahncard at the beginning of the interval $I_i$). Then, for all $i$ we have

\begin{equation*}
    \frac{\left(\Delta P\right)_{I_i}}{\mbox{cost}(\mathcal{A})_{I_i}} \leq \frac{\left\lceil\lambda \cdot \frac{B}{1-\beta}\right\rceil}{B+\beta \cdot \left\lceil\lambda \cdot \frac{B}{1-\beta}\right\rceil}\cdot \frac{e(\lambda)-\beta}{e(\lambda)-1}
\end{equation*}{}
\end{lem}{}
\begin{proof}
Assume that $m$ trips are requested during this interval $I_i$. Then we first have that $cost(\mathcal{A})_{I_i}=B+\beta m$ ($\mathcal A$ buys a Bahncard then pays a discounted price for every trip in the interval $I_i$).

As for Algorithm \ref{alg:LA Primal-Dual for Bahncard}, for each trip $j$, we are possibly in the first two cases: either $\sum_{t=t(j)-T}^{t(j)}x_t\geq 1$ in which case the increase in the primal is $\Delta P = \beta$ or in the second case in which case the increase in the primal is 
\begin{equation*}
    \Delta P = (1-\beta) \cdot \left(\sum_{t=t(j)-T}^{t(j)} x_t + \frac{1}{e(\lambda)-1} \right)+\beta \cdot \sum_{t=t(j)-T}^{t(j)} x_t + 1-\sum_{t=t(j)-T}^{t(j)} x_t= 1+\frac{1-\beta}{e(\lambda)-1}
\end{equation*}

We claim that the updates of the second case can happen at most $\left\lceil\lambda \cdot \frac{B}{1-\beta}\right\rceil$ times during interval $I_i$. To see this, denote by $S(l)$ the value of $\sum_{t'\geq t_i} x_{t'}$ after $l$ big updates in interval $I_i$. Note that once $\sum_{t'\geq t_i} x_{t'}\geq 1$, big updates cannot happen anymore. Hence all we need to prove is that $S\left(\left\lceil\lambda \cdot \frac{B}{1-\beta}\right\rceil\right)\geq 1$.

We prove by induction that $$S(k)\geq \frac{e(k \cdot (1-\beta)/B)-1}{e(\lambda)-1}$$

This is indeed true for $k=0$ as $S(0)$ is the value of $\sum_{t'\geq t_i} x_{t'}$ before any big update was made in $I_i$ hence $S(0)\geq 0$. Now assume this is the case for some $k$ and compute 

\begin{align*}
    S(k+1) &\geq \left(1+\frac{1-\beta}{B}\right) \cdot S(k) + \frac{1-\beta}{B}\cdot \frac{1}{e(\lambda)-1}\\
    &\geq \frac{\left(1+\frac{1-\beta}{B}\right)\cdot \left(e(k \cdot (1-\beta)/B) - 1\right)+\frac{1-\beta}{B}}{e(\lambda)-1}\\
    &\geq \frac{e((k+1) \cdot (1-\beta)/B)-1}{e(\lambda)-1}
\end{align*}{}
which concludes the induction.

Hence on interval $I_i$, the total increase in the cost of the solution can be bounded as follows 
$$\left(\Delta P\right)_{I_i} \leq \min\left\lbrace \left\lceil\lambda \cdot \frac{B}{1-\beta}\right\rceil, m \right\rbrace \cdot \left(1+\frac{1-\beta}{e(\lambda)-1}\right)+ \max\left\lbrace 0,\left(m- \left\lceil\lambda \cdot \frac{B}{1-\beta}\right\rceil\right)\right\rbrace \cdot \beta$$

One can see that the worst case possible for the ratio $\frac{\left(\Delta P\right)_{I_i}}{\mbox{cost}(\mathcal{A})_{I_i}}$ is obtained for $m=\left\lceil\lambda \cdot \frac{B}{1-\beta}\right\rceil$ and is bounded by 

$$\frac{\left(\Delta P\right)_{I_i}}{\mbox{cost}(\mathcal{A})_{I_i}} \leq 
 \frac{\left\lceil\lambda \cdot \frac{B}{1-\beta}\right\rceil \cdot \left(1+\frac{1-\beta}{e(\lambda)-1}\right)}{B+\beta \cdot \left\lceil\lambda \cdot \frac{B}{1-\beta}\right\rceil} = \frac{\left\lceil\lambda \cdot \frac{B}{1-\beta}\right\rceil}{B+\beta \cdot \left\lceil\lambda \cdot \frac{B}{1-\beta}\right\rceil}\cdot \frac{e(\lambda)-\beta}{e(\lambda)-1}$$
\end{proof}{}

We then consider times $t$ that do not belong to any interval $I_i$. More precisely, we upper bound the value $\left(\Delta P\right)_j$ that is the increase in cost of the primal solution due to trip $j$ such that $t(j)$ does not belong to any interval $I_i$. Note that in this case the prediction always pays a cost of $1$. 

\begin{lem}
\label{lem:consistency_2_Bahncard}
For any trip $j$ such that $t(j)\notin \bigcup_i I_i $, we have that 

$$\left(\Delta P\right)_{j} \leq \frac{e(1/\lambda)-\beta}{e(1/\lambda)-1} $$
\end{lem}{}
\begin{proof}
Note that Algorithm \ref{alg:LA Primal-Dual for Bahncard} pays either the cost of a small update which is $\frac{e(1/\lambda)-\beta}{e(1/\lambda)-1}$ or the cost of a minimal update which is $\beta$.
\end{proof}{}

For simplicity and better readability, we will formulate the final theorem of this section only for $\frac{B}{1-\beta}\rightarrow \infty$.

\begin{thm*}[Theorem \ref{thm:PDLA bahncard} restated]
    For any $\lambda \in (0,1]$, any $\beta \in [0,1]$ and ${\frac{B}{1-\beta} \xrightarrow[]{} \infty }$, we have the following guarantees on any instance $\mathcal I$
  $$ \textit{cost}_{\tiny{\mathcal{PDLA}}}(\mathcal A, \mathcal I,\lambda) \leq \min \left\lbrace \frac{\lambda}{1-\beta +\lambda \beta }\cdot \frac{e^{\lambda} - \beta}{e^{\lambda} - 1}\cdot S(\mathcal A, I), \frac{e^{\lambda}-\beta}{e^{\lambda}-1}\cdot \OPT \right\rbrace$$
\end{thm*}
\begin{proof}
By taking the limit in Lemma \ref{lem:robustness Bahncard}, we see that the cost of the solution output by Algorithm \ref{alg:LA Primal-Dual for Bahncard} is at most $\frac{e^\lambda-\beta}{e^\lambda-1}\cdot \OPT$ which proves the second bound in the theorem.

For the first bound, note that we can write the final cost of the solution as 
$$\textit{cost}_{\tiny{\mathcal{PDLA}}}(\mathcal A, \mathcal I,\lambda) = \Delta P = \sum_i \left(\Delta P \right)_{I_i}+\sum_{j: t(j)\notin \bigcup_i I_i} \left(\Delta P \right)_j $$

By taking the limit in Lemma \ref{lem:consistency_1_Bahncard} we get that 
$$\sum_i \left(\Delta P \right)_{I_i}\leq \frac{\lambda}{1-\beta + \beta\lambda}\cdot \frac{e^\lambda - \beta}{e^\lambda - 1} \cdot \sum_i \textit{cost}(\mathcal{A})_{I_i}$$
and by taking the limit in Lemma \ref{lem:consistency_2_Bahncard}, we get that
$$\sum_{j: t(j)\notin \bigcup_i I_i} \left(\Delta P \right)_j \leq \frac{e^{1/\lambda}-\beta}{e^{1/\lambda}-1}\cdot \sum_{j: t(j)\notin \bigcup_i I_i} \textit{cost}(\mathcal A)_j$$

By using Lemma \ref{lem:inequalities} (inequality \eqref{eqn:lemma12_3}), we see that 
$$\max\left\lbrace \frac{\lambda}{1-\beta + \beta\lambda}\cdot \frac{e^\lambda - \beta}{e^\lambda - 1}, \frac{e^{1/\lambda}-\beta}{e^{1/\lambda}-1}\right\rbrace = \frac{\lambda}{1-\beta + \beta\lambda}\cdot \frac{e^\lambda - \beta}{e^\lambda - 1}$$
which ends the proof.

\end{proof}

We finish this section by proving that a fractional solution can be rounded online into a randomized integral solution. The expected cost of the rounded instance will be equal to the cost of the fractional solution. If the rounding is very similar to the existing rounding of \citet{ad_auctionsESA} for ski rental or TCP acknowledgement, we still include it here for completeness as the Bahncard problem was never solved in a primal-dual way. The argument is summarized in the following lemma.
\begin{lem}
Given a fractional solution $(x,d,f)$ to the Bahncard problem, it can be rounded online into an integral solution of expected cost equal to the fractional cost of $(x,d,f)$.
\end{lem}
\begin{proof}
Choose some real number $p$ uniformly at random in the interval $[0,1]$. Then arrange the variables $x_t$ on the real line (i.e. iteratively as follows, each time $t$ takes an interval $I_t$ of length $x_t$ right after the interval taken by $x_{t-1}$). Then buy a Bahncard at every time $t$ such that the interval corresponding to time $t$ contains the real number $p+k$ for some integer $k$. We check first that the expected buying cost is
\begin{equation*}
    B\cdot \sum_t \mathbbm{E}\left(\mathds{1}_{p+k\in I_t}\right) = B\cdot \sum_t x_t
\end{equation*}

Next, to compute the total expected price of the tickets, notice that if a ticket was bought in the previous $T$ time steps, we can pay a discounted price, otherwise we need to pay the full price of $1$. For a trip $j$, the probability that a ticket was bought in the previous $T$ time steps is at least $\sum_{t=t(j)-T}^{t(j)}x_t$. Hence with probability at least $\sum_{t=t(j)-T}^{t(j)}x_t\geq d_j$ we pay a price of $\beta$ and with probability $1-d_j\leq f_j$ we pay a price of $1$ which ends the proof.
\end{proof}

\section{Missing proofs for TCP}
\label{sec:TCP appendix}
\subsection{Plots of instances}

We briefly show in Figures \ref{fig:poisson_instance}, \ref{fig:pareto_instance}, and \ref{fig:poisson_it_instance} how typical instances under various distributions look like.

\begin{figure}[H]
    \centering
    \includegraphics[scale=0.5]{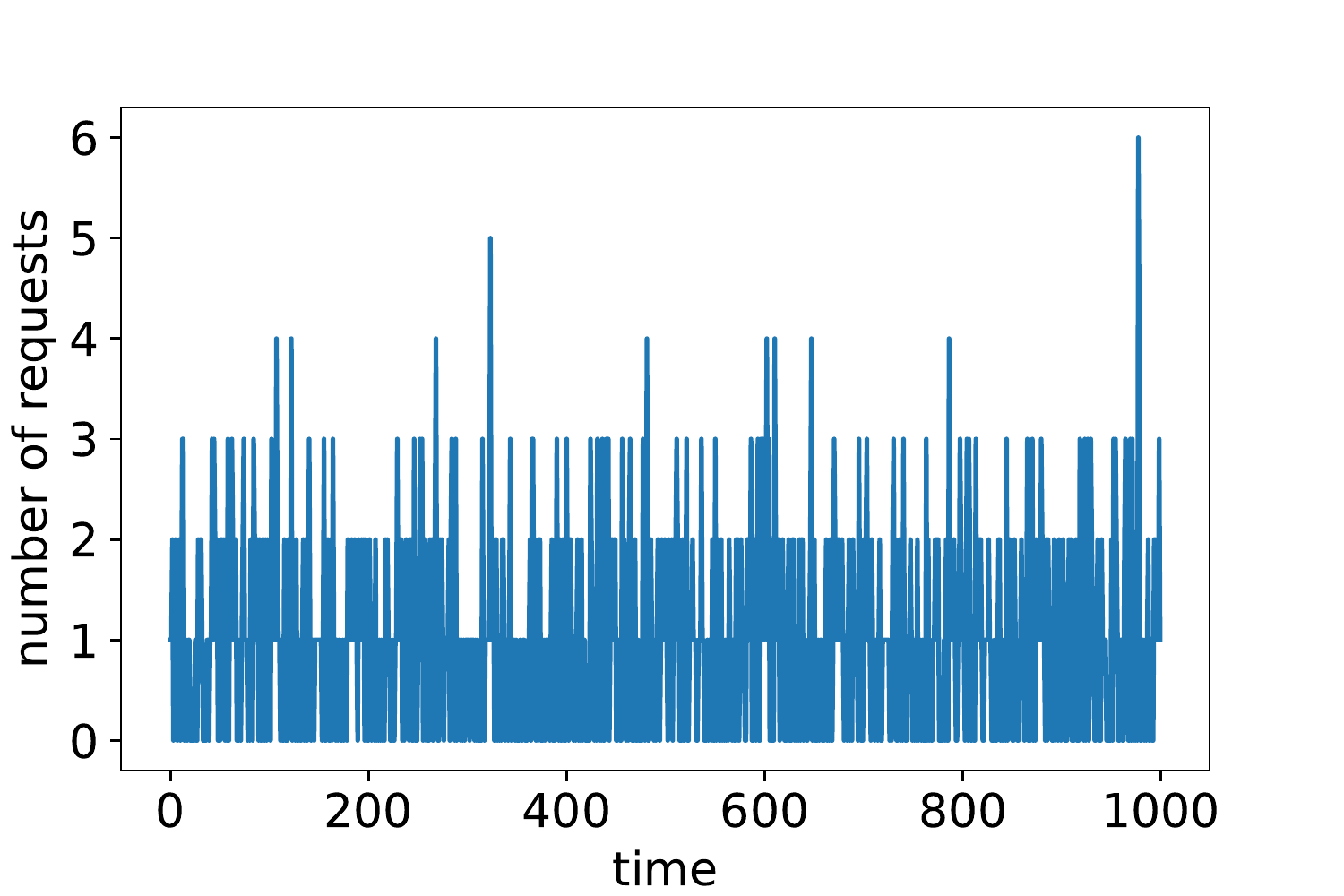}
    \caption{Typical instance under Poisson distribution}
    \label{fig:poisson_instance}
\end{figure}

\begin{figure}[H]
    \centering
    \includegraphics[scale=0.5]{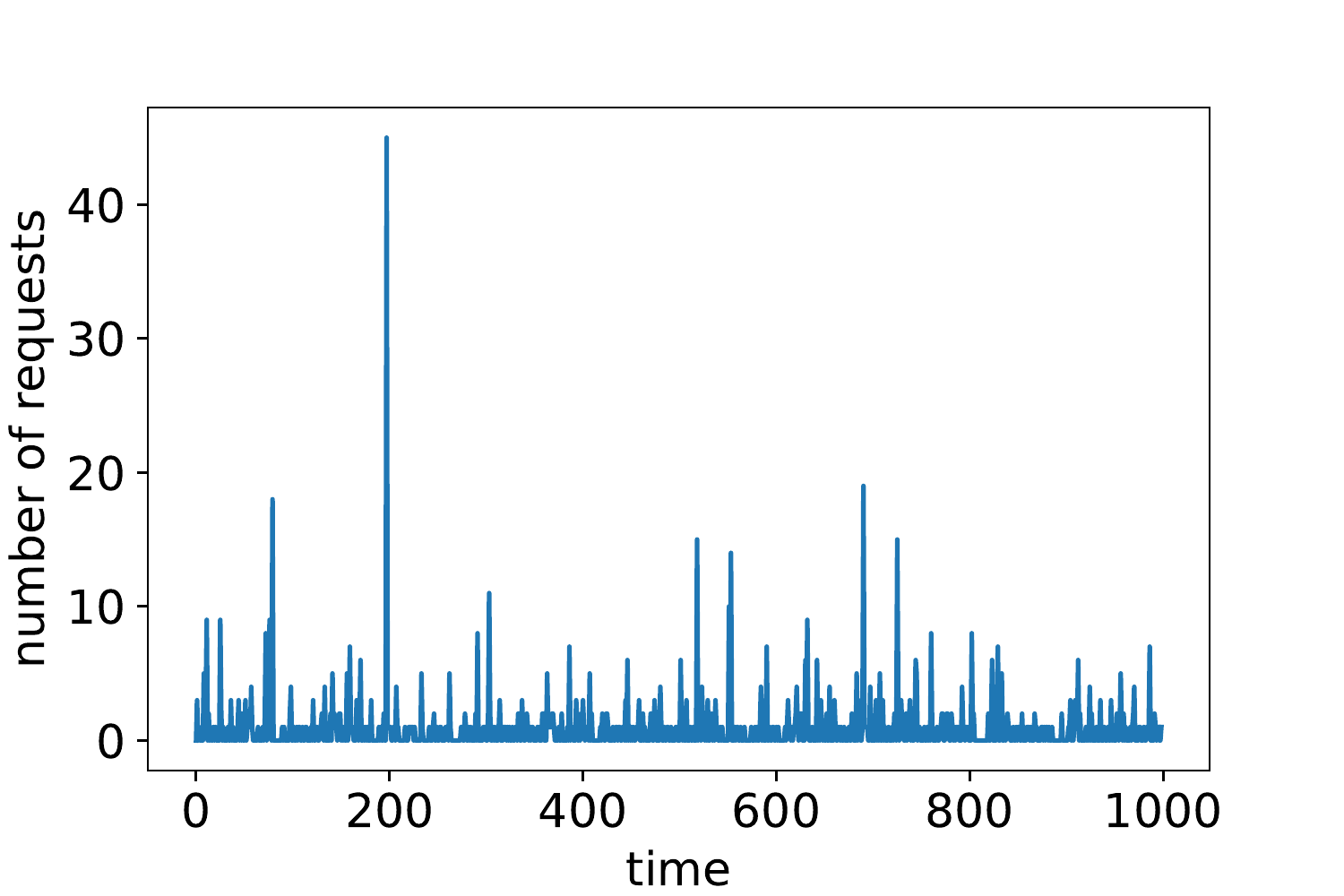}
    \caption{Typical instance under Pareto distribution}
    \label{fig:pareto_instance}
\end{figure}

\begin{figure}[H]
    \centering
    \includegraphics[scale=0.5]{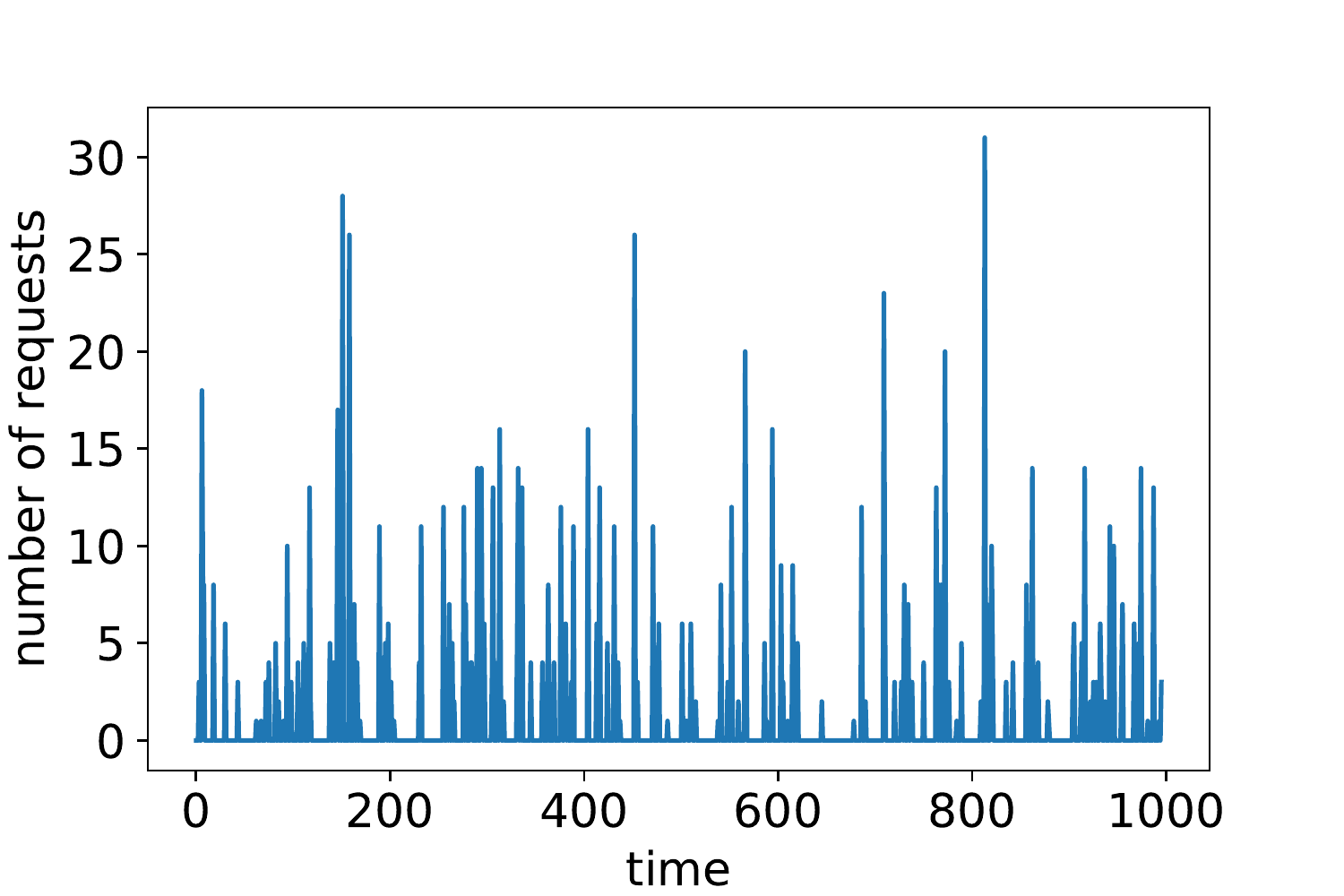}
    \caption{Typical instance under iterated Poisson distribution}
    \label{fig:poisson_it_instance}
\end{figure}

\subsection{Theoretical analysis}
Recall that we define in this section \(e(z)=(1+1/d)^{z\cdot d}\) which will be roughly equal to \(e^z\) for big \(d\). The big updates are then the updates where \(z\) is set to \(\lambda\) and during a small update, \(z\) is set to \(1/\lambda\).

We first analyze the consistency of this algorithm. To this end denote by \(n_\mathcal{A}\) the number of acknowledgements sent by \(\mathcal{A}\) and by \(\textit{latency}\left(\mathcal{A} \right)\) the latency paid by the prediction \(\mathcal{A}\).
\begin{lem} 
\label{lem:consistency_TCP}
For any \(\lambda\in (0,1]\), \(d>0\),
\[\textit{c}_{\tiny{\mathcal{PDLA}}}(\mathcal A, \I,\lambda)\leq n_{\mathcal{A}}  \cdot \frac{1}{d}\cdot \frac{\left\lceil \lambda d\right\rceil}{1-e(-\lambda)} + \textit{latency}(\mathcal{A})\cdot  \frac{1}{1-e(-1/\lambda)}\]
\end{lem}
\begin{proof}
We will use a charging argument to analyze the performance of Algorithm \ref{alg:LA Primal-Dual for TCP ack}. Note that for a small update, the increase in cost of the fractional solution is 

\[\Delta P = \frac{1}{d}\left(1 - \sum_{k = t(j)}^t x_k\right)+\frac{1}{d}\cdot \left(\sum_{k = t(j)}^t x_k + \frac{1}{e(1/\lambda)-1} \right) = \frac{1}{d}\cdot \frac{1}{1-e(-1/\lambda)}\]

However, for every small update that is made, it must be that \(\mathcal{A}\) pays a latency of at least \(\frac{1}{d}\). Hence the total cost of small updates made by Algorithm \ref{alg:LA Primal-Dual for TCP ack} is at most \(\textit{latency}(\mathcal{A}) \cdot \frac{1}{1-e(-1/\lambda)}\).

Secondly we bound the total cost of big updates of our algorithm. Let \(t_0\) be a time at which \(\mathcal{A}\) sends an acknowledgment. Let \(Y\) be the set of big updates made because of jobs \(j\) that are acknowledged at time \(t_0\) by \(\mathcal{A}\) (these big updates are hence made at some time \(t\geq t_0\)). We claim that \(|Y|\leq \left\lceil \lambda d\right\rceil\).

To prove this denote by \(S(l)\) the value of \(\sum_{k=t_0}^{+\infty} x_k\) after \(l\) such big updates (there might be small updates influencing this value but only to make it bigger). Notice that once \(\sum_{k=t_0}^{+\infty} x_k\geq 1\) there is no remaining update in \(Y\). We prove by induction that 
\[ S(l)\geq \frac{(1+1/d)^l - 1}{(1+1/d)^{\lambda d}-1}\]

This is clear for \(l=0\) as \(S(0)\geq 0\). Now assume this is the case for some value \(l\) and apply a big update at time \(t\) for job \(j\) to get

\begin{align*}
    S(l+1)&=S(l)+\frac{1}{d}\cdot \left(\sum_{k=t(j)}^t x_k + \frac{1}{e(\lambda)-1}\right)\\
    &\geq S(l)\cdot (1+1/d) + \frac{1}{d(e(\lambda)-1)}\\
    &= \frac{(1+1/d)^{l+1} - 1-1/d}{(1+1/d)^{\lambda d}-1}+ \frac{1/d}{(e(\lambda)-1)}\\
    &=\frac{(1+1/d)^{l+1} - 1-1/d}{(1+1/d)^{\lambda d}-1}+ \frac{1/d}{(1+1/d)^{\lambda d}-1}\\
    &=\frac{(1+1/d)^{l+1} - 1}{(1+1/d)^{\lambda d}-1}
\end{align*}{}

Where the second inequality comes from noting that since we are considering an update due to a request $j$ acknowledged at time $t_0$ by the predicted solution, it must be that $t(j) \leq t_0$ and  $\sum_{k=t(j)}^t x_k \geq \sum_{k=t_0}^t x_k$.
Hence we get that \(S\left(\left\lceil \lambda d\right\rceil \right)\geq 1\) which implies that \(|Y|\leq \left\lceil \lambda d\right\rceil\). 

By a similar calculation as for the small update case, we have that the cost of a big update is 

\begin{equation*}
    \Delta P = \frac{1}{d}\cdot \frac{1}{1-e(-\lambda)}
\end{equation*}

Hence the total cost of these updates in \(Y\) is charged to the acknowledgement that \(\mathcal{A}\) pays at time \(t_0\) to finish the proof.
\end{proof}

Taking the limit \(d\rightarrow +\infty\) we get the following corollary:

\begin{cor}
\label{cor:TCP_consistency_asymptotic}
For any \(\lambda\in (0,1]\) and taking \(d\rightarrow +\infty\), we have that
\[\textit{c}_{\tiny{\mathcal{PDLA}}}(\mathcal A, \I,\lambda) \leq n_\mathcal{A}\cdot \frac{\lambda}{1-e^{-\lambda}}+\textit{latency}(\mathcal{A})\cdot \frac{1}{1-e^{-1/\lambda}} \]
\end{cor}{}

We then prove robustness of the algorithm with the following lemmas.

\begin{lem}
\label{lem:TCP_feasible}
Let \(y\) be the dual solution produced by Algorithm \ref{alg:LA Primal-Dual for TCP ack}. Then \(\frac{y}{1+1/d}\) is feasible.
\end{lem}{}
\begin{proof}
Notice that the constraints of the second type (i.e. \(0 \leq y_{jt} \leq 1/d\)) are always satisfied since \(0<\lambda\leq 1\). We now check that the second constraints are almost satisfied (within some factor \((1+1/d)\)). Fix a time \(t\in T\) and consider the corresponding constraint: 
\[\sum_{j | t \geq t(j)} \sum_{t' \geq t} y_{jt} \leq 1\]

Note that for a small update for some job \(j\) such that \(t(j)\leq t\) the sum above increases by \(\lambda/d\) while it increases by \(1/d\) for a big update. Notice that once we have that \(\sum_{t'\geq t} x_{t'}\geq 1\), no more such update will be performed. Denote by $S$ the value of this sum.

Notice that for a big update, the sum $S$ becomes $\left(1+\frac{1}{d} \right)\cdot S + \frac{1}{d((1+1/d)^{\lambda d} -1)}$. Similarly, for a small updates it becomes $\left(1+\frac{1}{d} \right)\cdot S + \frac{1}{d((1+1/d)^{d/\lambda} -1)}$.

Hence, if we denote by \(s\) the number of small updates in this sum and by \(b\) the number of big updates, by Lemma \ref{lem:sequence_lemma} we have that if \(\lambda s + b \geq d\) then \(\sum_{t'\geq t} x_{t'}\geq 1\). This directly implies that the value of $\sum_{j | t \geq t(j)} \sum_{t' \geq t} y_{jt}$ is at most \(1+1/d\) at the end of the algorithm (each update in the dual is of value at most \(1/d\)).

Therefore scaling down all $y_{jt}$ by a multiplicative factor of $1+1/d$ yields a feasible solution to the dual.
\end{proof}{}

\begin{lem}
\label{lem:TCP_robustness}
For \(d\rightarrow +\infty\), Algorithm \ref{alg:LA Primal-Dual for TCP ack} outputs a solution of cost at most \(\frac{1}{1-e^{-\lambda}}\cdot \OPT\)
\end{lem}{}
\begin{proof}
We first compare the increase \(\Delta P\) in the primal value to the increase \(\Delta D\) in the dual value at every update. We claim that for every update we have 
\begin{equation*}
    \frac{\Delta P}{\Delta D}\leq \frac{1}{1-e(-\lambda)}
\end{equation*}{}
In the case of a big update we directly have \(\Delta P=\frac{1}{d}\left(1+\frac{1}{e(\lambda)-1} \right)=\frac{1}{d}\cdot\frac{1}{1-e(-\lambda)}\) and \(\Delta D = \frac{1}{d}\). In the case of a small update we have \(\Delta D = \frac{\lambda}{d}\) and \(\Delta P = \frac{1}{d}\left(1+\frac{1}{e(1/\lambda)-1} \right)=\frac{1}{d}\cdot\frac{1}{1-e(-1/\lambda)}\) and we conclude applying Lemma \ref{lem:inequalities} (inequality \eqref{eqn:lemma12_1}) that we always have

\begin{equation*}
    \frac{\Delta P}{\Delta D} \leq \frac{1}{1-e(-\lambda)}
\end{equation*}

By lemma \ref{lem:TCP_feasible}, \(\frac{y}{1+1/d}\) is a feasible solution. Hence taking \(d \rightarrow +\infty\) together with the previous remark and weak duality we get the result.
\end{proof}{}

Combining Lemma \ref{cor:TCP_consistency_asymptotic} and Lemma \ref{lem:TCP_robustness} yields Theorem \ref{thm:tcp_main}.

\section{Optimality bound}
\label{sec: Optimality appendix}
\newcommand{\tend}{t_{\textit{end}}}
\primelemma*
\begin{proof}
For simplicity, we will consider the ski-rental problem in the continuous case which corresponds to the behaviour of the discrete version when $B \xrightarrow[]{} \infty$. In this problem, the cost of buying is $1$ and a randomized algorithm has to define a (buying) probability distribution $\{ p_t\}_{t \geq 0}$. Moreover, consider the case where the true number of vacation days $\tend \in [0,1] \cup (2,\infty)$. In such a case we can assume w.l.o.g. that $p_t = 0,\forall t >1$. Indeed moving buying probability mass from any $p_t, t >1$ to $p_1$ does not increase the cost of the randomized algorithm. Assume now that the prediction suggests us that the end of vacations is at $\hat{t}_{\textit{end}} >  2$, thus the optimal offline solution, if the prediction is correct, is to buy the skis in the beginning for a total cost of $1$. Since the algorithm has to define a probability distribution in $[0,1]$,  $\{ p_t\}$ needs to satisfy the equality constraint $\int_0^1 p_t dt = 1 $. Moreover, note that when the prediction is correct, i.e. $t_{\textit{end}} > 2$, the LA algorithm suffers an expected cost of $\int_0^1 (t+1)p_t dt$ while the optimum offline has a cost of $1$. Thus the consistency requirement forces the distribution to satisfy the inequality $\int_0^1 (t+1)p_t dt \leq \frac{\lambda}{1 - e^{-\lambda}}$. 
Now assume that the best possible LA algorithm is $c$-robust. If $t_{\textit{end}} \leq 1$ then the LA algorithm's cost is ${\int_0^{\tend}  (t+1)p_tdt  + \tend \int_{\tend}^1 p_t dt}$ while the optimum offline cost is $\tend$. Thus, due to $c$-robustness we have that for every $t' \in [0,1]$, ${\int_0^{t'}  (t+1)p_tdt  + t' \int_{t'}^1 p_t dt \leq c t'}$. We calculate the best possible robustness $c$ with the following LP:
\begin{figure}[H]
    \centering
    \caption{Primal Robustness for ski-rental problem.}
    \label{fig:Primal Optimal Robustness for Ski Rental}
    \begin{tabular}{|c|}
    \hline
    \textbf{Primal} \\
    \hline
    \text{minimize} $c$ \\
    \text{subject to:} $\quad \int_0^1 p_t dt = 1$\\
     $\quad \int_0^1 (t+1)p_t dt \leq \frac{\lambda}{1 - e^{-\lambda}}$ \\
     $\quad {\int_0^{t'}  (t+1)p_tdt  + t' \int_{t'}^1 p_t dt \leq c t'} \quad \forall t' \in [0,1]$ \\
     $\quad p_t \geq 0 \quad \forall t' \in [0,1]$\\
    \hline
\end{tabular}{}
\end{figure}{}
To lower bound the best possible robustness $c$ we will present a feasible solution to the dual of \ref{fig:Primal Optimal Robustness for Ski Rental}. The dual variables $\lambda_d$ and $\lambda_c$ correspond respectively to the first and second primal constraints in Figure \ref{fig:Primal Optimal Robustness for Ski Rental}. The dual variables $\lambda_t$, $\forall t \in [0,1]$ correspond to the robustness constraints described in the third line of the primal.

The corresponding dual is:

\begin{figure}[H]
    \centering
    \caption{Dual Robustness for ski-rental problem.}
    \label{fig:Dual Optimal Robustness for Ski Rental}
    \begin{tabular}{|c|}
    \hline
    \textbf{Dual} \\
    \hline
    \text{maximize} $\lambda_d  -\lambda_c \cdot \frac{\lambda}{1 - e^{-\lambda}}$ \\
    \text{subject to:} $\quad \int_0^1 t \lambda_t dt \leq 1$\\
     $\quad  \lambda_d -(t'+1) \lambda_c \leq \int_0^{t'} t \lambda_t dt  + (t'+1)\int_{t'}^1 \lambda_t dt \quad \forall t' \in [0,1]$\\
     $\lambda_c, \lambda_t \geq 0 \quad \forall t \in [0,1]$ \\
    \hline
\end{tabular}{}
\end{figure}{}
Let $K = \frac{1}{1 - \lambda e^{-\lambda}  - e^{-\lambda}}$. Then, $\lambda_t = K \cdot e^{-t}  \cdot \mathbbm{1} \{ t \leq \lambda\}  $ , $\lambda_d = K$ and $\lambda_c = K \cdot e^{-\lambda}$.

We first prove that this dual solution is feasible. For the first constraint notice that
\begin{equation*}
    \int_0^1 t\lambda_t dt = K\cdot \int_0^\lambda te^{-t}dt=K\cdot \left(1-(\lambda+1)e^{-\lambda}\right)=1
\end{equation*}

For the second type of constraint first in the case $t'>\lambda$ we get
\begin{equation*}
    \int_0^{t'} t \lambda_t dt  + (t'+1)\int_{t'}^1 \lambda_t dt = \int_0^{\lambda} t \lambda_t dt = 1
\end{equation*}
and we note that 
\begin{equation*}
    \lambda_d -(t'+1) \lambda_c \leq \lambda_d -(\lambda+1) \lambda_c = K\cdot \left(1-(\lambda+1)e^{-\lambda}\right)=1
\end{equation*}
hence these constraints are satisfied.

In the second case $t'\leq \lambda$, we have that

\begin{align*}
    \int_0^{t'} t \lambda_t dt  + (t'+1)\int_{t'}^1 \lambda_t dt &= K\cdot \left(\int_0^{t'} t e^{-t} dt + (t'+1)\int_{t'}^\lambda e^{-t}dt\right)\\
    &= K\cdot \left(1-(t'+1)e^{-t'}+(t'+1)(e^{-t'}-e^{-\lambda})\right)\\
    &= K\cdot \left(1-(t'+1)e^{-\lambda}\right)\\
    &=\lambda_d -(t'+1) \lambda_c
\end{align*}

which proves that these constraints are also satisfied. Hence this dual solution is feasible. Finally note that the cost of this dual solution is
\begin{align*}
    \lambda_d  -\lambda_c \cdot \frac{\lambda}{1 - e^{-\lambda}} &= K\cdot \left(1-\frac{\lambda}{1 - e^{-\lambda}} \cdot e^{-\lambda} \right)\\
    &=  K\cdot \frac{1-e^{-\lambda}-\lambda e^{-\lambda}}{1-e^{-\lambda}} = \frac{1}{1-e^{-\lambda}}
\end{align*}

By weak duality, we conclude that the best robustness cannot be better than $\frac{1}{1-e^{-\lambda}}$
\end{proof}

\section{Technical lemmas}
\label{Technical lemmas}
A few inequalities that will be useful:

\begin{lem}
\label{lem:inequalities}
    For any $d>0$, any $0<\lambda\leq 1$, and any $\beta\in [0,1]$, we have:
    \begin{equation}
    \label{eqn:lemma12_1b}
        \frac{\lambda}{1-e^{-\lambda}} \geq \frac{1}{1-e^{-1/\lambda}}
    \end{equation}{}
    \begin{equation}
    \label{eqn:lemma12_1}
        \frac{\lambda}{1-(1+1/d)^{-\lambda d}} \geq \frac{1}{1-(1+1/d)^{-d/\lambda}}
    \end{equation}{}
    \begin{equation}
    \label{eqn:lemma12_2b}
        \frac{1}{e^{\lambda}-1} \geq \frac{\frac{1-\lambda}{\lambda} \cdot e^{1/\lambda}+1}{ e^{1/\lambda}-1}
    \end{equation}{}
    \begin{equation}
    \label{eqn:lemma12_2}
        \frac{1}{(1+1/d)^{\lambda d}-1} \geq \frac{\frac{1-\lambda}{\lambda} \cdot (1+1/d)^{d/\lambda}+1}{ (1+1/d)^{d/\lambda}-1}
    \end{equation}{}
    \begin{equation}
    \label{eqn:lemma12_3}
        \frac{\lambda}{1-\beta + \beta\lambda}\cdot \frac{e^\lambda - \beta}{e^\lambda - 1} \geq \frac{e^{1/\lambda}-\beta}{e^{1/\lambda}-1}
    \end{equation}{}
    \begin{equation}
    \label{eqn:lemma12_4}
        (\lambda+\beta - \beta\lambda)\cdot \frac{e^\lambda - \beta}{e^\lambda - 1} \geq \frac{e^{1/\lambda}-\beta}{e^{1/\lambda}-1}
    \end{equation}{}
\end{lem}
\begin{proof}
Since the formal proof of \eqref{eqn:lemma12_1b} and \eqref{eqn:lemma12_2b} seems to require heavy calculations and that they are easy to check on computer we will only give a proof by a plot (see Figures \ref{fig:f} and \ref{fig:g}). For \ref{fig:g}, note that $\eqref{eqn:lemma12_2b}\iff \frac{1}{e^\lambda-1}-\frac{1-\lambda}{\lambda}-\frac{1}{\lambda}\cdot \frac{e^{-1/\lambda}}{1-e^{-1/\lambda}} \geq 0$.
\begin{figure}[H]
\centering
\begin{subfigure}{.5\textwidth}
  \centering
  \includegraphics[width=1.1\linewidth]{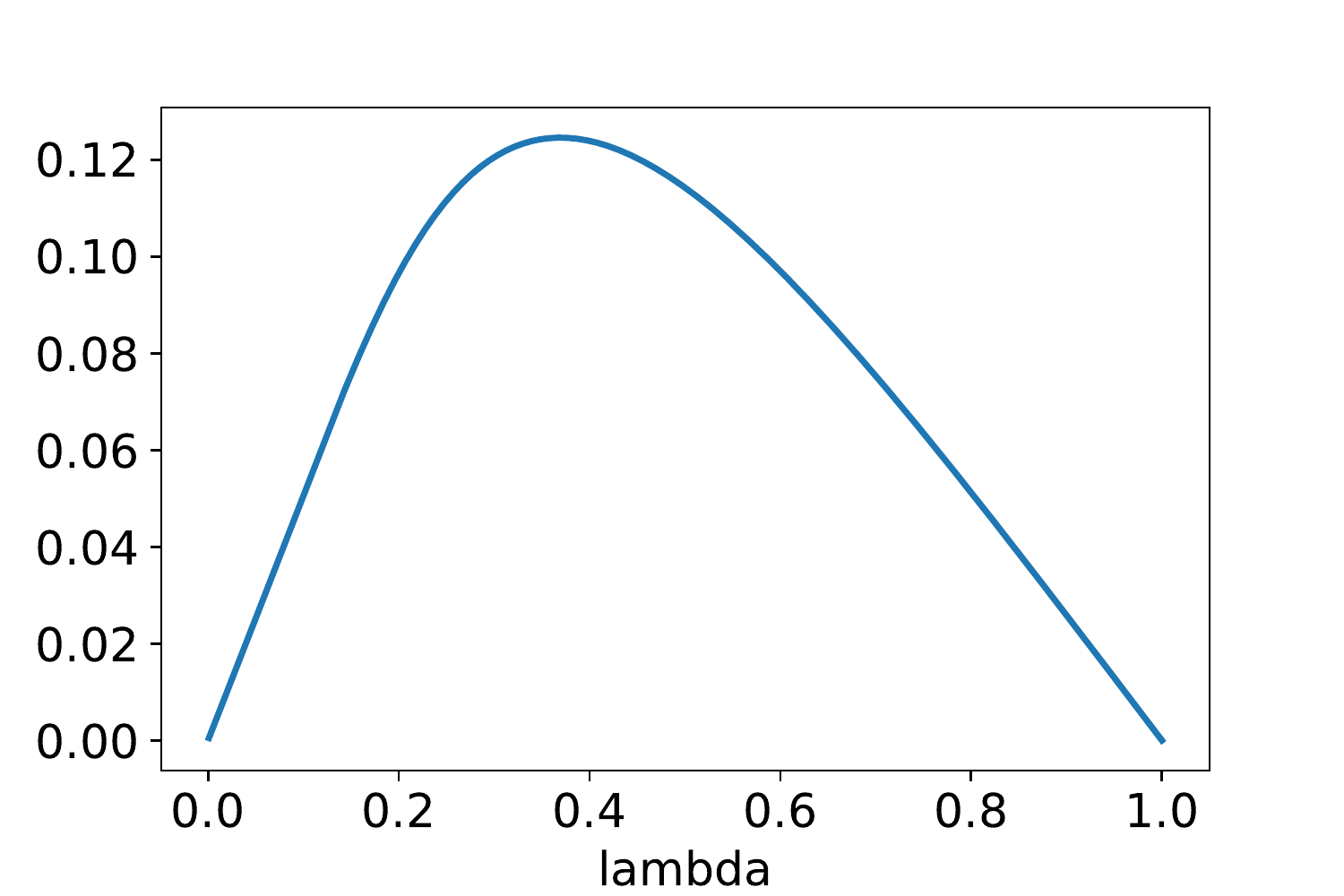}
  \caption{Plot of $\frac{\lambda}{1-e^{-\lambda}}-\frac{1}{1-e^{-1/\lambda}}$ }
  \label{fig:f}
\end{subfigure}%
\begin{subfigure}{.5\textwidth}
  \centering
  \includegraphics[width=1.1\linewidth]{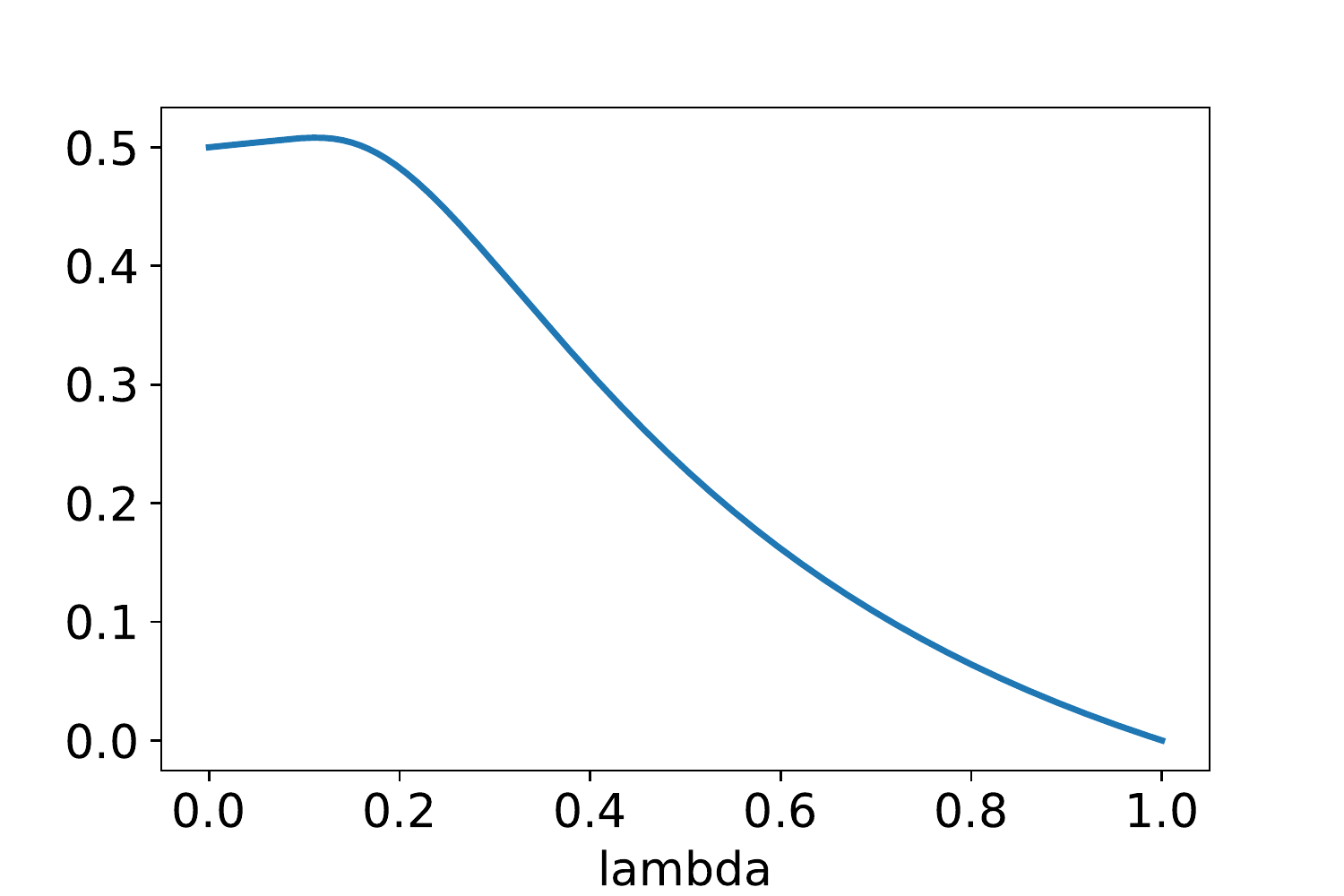}
  \caption{Plot of $\frac{1}{e^\lambda-1}-\frac{1-\lambda}{\lambda}-\frac{1}{\lambda}\cdot \frac{e^{-1/\lambda}}{1-e^{-1/\lambda}}$ }
  \label{fig:g}
\end{subfigure}
\caption{Plots for \eqref{eqn:lemma12_1b} and \eqref{eqn:lemma12_2b}}
\label{fig:plots}
\end{figure}
We now prove that inequality \eqref{eqn:lemma12_1b} implies inequality \eqref{eqn:lemma12_1}. For this end notice that we can write $(1+1/d)^d=e^x$ for some $x\in (0,1)$ since $(1+1/d)^d\in (1,e)$ for all $d>0$. We prove that for any $x\in (0,1]$
\begin{align*}
    \frac{\lambda \left(1-e^{-x/\lambda} \right)}{1-e^{-x\lambda}}\geq \frac{\lambda\left(1-e^{-1/\lambda} \right)}{1-e^{-\lambda}}
\end{align*} which will imply our claim since by inequality \eqref{eqn:lemma12_1b} the right hand side is bigger than $1$. First note this is equivalent to prove that 
\begin{align*}g_\lambda(x)=(1-e^{-\lambda})\cdot (1-e^{-x/\lambda})- (1-e^{-1/\lambda})\cdot (1-e^{-x\lambda}) \geq 0\end{align*}
Taking the derivative of $g_\lambda(x)$ we obtain 
\begin{align*}g'_\lambda(x) = \frac{1-e^{-\lambda}}{\lambda}\cdot e^{-x/\lambda}-\lambda (1-e^{-1/\lambda})\cdot e^{-x\lambda}\end{align*}
hence we can write 
\begin{align*}g'_\lambda(x)\geq 0 \iff e^{x\left(\lambda-1/\lambda\right)} \geq \lambda^2 \cdot \frac{1-e^{-1/\lambda}}{1-e^{-\lambda}}\end{align*}
Notice that the left hand side in this inequality is decreasing because $\lambda \in (0,1]$. Also notice that $g_\lambda(0)=g_\lambda(1)=0$. These two facts together imply that $g_\lambda$ is first increasing for $x\in (0,c]$ then decreasing for $x\in (c,1]$ for some unknown $c$. In particular, we indeed have that $g_\lambda(x)\geq 0$ which ends the proof of inequality \eqref{eqn:lemma12_1}.

Similarly, we prove that inequality \eqref{eqn:lemma12_2b} implies inequality \eqref{eqn:lemma12_2}. Again we write $(1+1/d)^d=e^x$ for some $x\in (0,1)$. We first rewrite inequality \eqref{eqn:lemma12_2}.

\begin{align*}
    \eqref{eqn:lemma12_2} &\iff \frac{1}{e^{\lambda x}-1} \geq \frac{\frac{1-\lambda}{\lambda}\cdot e^{x/\lambda}+1}{e^{x/\lambda}-1}\\
    &\iff \frac{1}{e^{\lambda x}-1} \geq \frac{\frac{1-\lambda}{\lambda}\cdot (e^{x/\lambda}-1)+\frac{1}{\lambda}}{e^{x/\lambda}-1}\\
    &\iff \lambda (e^{x/\lambda}-1)\geq (1-\lambda) (e^{x/\lambda}-1)(e^{\lambda x}-1)+(e^{\lambda x}-1)\\
    &\iff \lambda (e^{x/\lambda}-1) - (1-\lambda) (e^{x/\lambda}-1)(e^{\lambda x}-1) - (e^{\lambda x}-1) \geq 0
\end{align*}

Define the following function $h_{\lambda}(x)=\lambda (e^{x/\lambda}-1) - (1-\lambda) (e^{x/\lambda}-1)(e^{\lambda x}-1) - (e^{\lambda x}-1)$. One can first compute:

\begin{align*}
    h'_\lambda (x) &= e^{x/\lambda}-(1-\lambda)\cdot \left( \lambda e^{\lambda x}(e^{x/\lambda}-1) + \frac{1}{\lambda} e^{x/\lambda}(e^{\lambda x}-1)\right)-\lambda e^{\lambda x}\\
    &= e^{x/\lambda}-\lambda e^{\lambda x}- (1-\lambda)\cdot \left((\lambda+1/\lambda)e^{x(\lambda+1/\lambda)} -\lambda e^{\lambda x} - \frac{1}{\lambda} e^{x/\lambda} \right)\\
    &= e^{x/\lambda}\cdot \left(1+\frac{1-\lambda}{\lambda} \right) + e^{\lambda x}\cdot \left(-\lambda + \lambda (1-\lambda) \right) -e^{x(\lambda+1/\lambda)}\cdot  (1-\lambda)\cdot \left(\lambda + \frac{1}{\lambda} \right)\\
    &= \frac{e^{x/\lambda}}{\lambda} -\lambda^2 e^{\lambda x} - \frac{e^{x(\lambda+1/\lambda)}}{\lambda}\cdot (1-\lambda)\cdot (\lambda^2+1)
\end{align*}

Hence we can rewrite 

\begin{align*}
    h'_\lambda (x) \geq 0 &\iff \frac{e^{x/\lambda}}{\lambda} -\lambda^2 e^{\lambda x} - \frac{e^{x(\lambda+1/\lambda)}}{\lambda}\cdot (1-\lambda)\cdot (\lambda^2+1) \geq 0\\
    &\iff e^{x/\lambda} - \lambda^3 e^{\lambda x}- e^{x(\lambda+1/\lambda)}\cdot (1-\lambda)\cdot (\lambda^2+1) \geq 0\\
    &\iff 1 - \lambda^3 e^{ x(\lambda-1/\lambda)}- e^{x\lambda}\cdot (1-\lambda)\cdot (\lambda^2+1) \geq 0
\end{align*}
Let us define $i_\lambda (x) = 1 - \lambda^3 e^{ x(\lambda-1/\lambda)}- e^{x\lambda}\cdot (1-\lambda)\cdot (\lambda^2+1)$ and we derive

\begin{align*}
    i'_\lambda(x) = -\lambda^3 \cdot (\lambda-1/\lambda)\cdot e^{ x(\lambda-1/\lambda)}-\lambda e^{\lambda x}\cdot (1-\lambda)\cdot (\lambda^2+1)
\end{align*}

We can now notice that 
\begin{align*}
    i'_\lambda(x) \geq 0 &\iff -\lambda^3 \cdot (\lambda-1/\lambda)\cdot e^{ x(\lambda-1/\lambda)}-\lambda e^{\lambda x}\cdot (1-\lambda)\cdot (\lambda^2+1)\geq  0\\
    &\iff  -\lambda^3 \cdot (\lambda-1/\lambda)\cdot e^{-x/\lambda}-\lambda (1-\lambda)\cdot (\lambda^2+1)\geq 0\\
    &\iff \lambda^3 \cdot (1/\lambda-\lambda)\cdot e^{-x/\lambda}-\lambda (1-\lambda)\cdot (\lambda^2+1)\geq 0
\end{align*}
Since the left hand side is decreasing as $x$ increases we only need to check one extreme value which is $i'_\lambda(0)$.
We write 
\begin{align*}
    i'_\lambda(0)\leq 0 &\iff \lambda^3 \cdot (1/\lambda-\lambda) -\lambda \cdot (1-\lambda)\cdot (\lambda^2+1) \leq 0\\
    &\iff \lambda^2-\lambda^4 - (\lambda^3+\lambda -\lambda^4-\lambda^2)\leq 0\\
    &\iff -\lambda^3 +2\lambda^2 -\lambda \leq 0\\
    &\iff -\lambda \cdot (\lambda-1)^2 \leq 0
\end{align*}
hence we always have $i'_\lambda(0)\leq 0$.

Therefore we get that $i'_\lambda(x)\leq 0$ for all $x$ and $\lambda$. Note that $i_\lambda(0)=1-\lambda^3-(1-\lambda)(\lambda^2+1)=1-\lambda^3-\lambda^2-1+\lambda^3+\lambda= \lambda -\lambda^2 \geq 0$. Therefore we get that $h_\lambda$ is first positive on some interval $[0,c]$ and then negative for $x\in [c,\infty)$. Therefore $h_\lambda$ is first increasing then decreasing. Notice that $h_\lambda(0)=0$ and $h_\lambda(1)\geq 0$ by inequality \eqref{eqn:lemma12_2b}. Hence inequality \eqref{eqn:lemma12_2} is true for all $x\in [0,1]$ which concludes the proof. 

Finally, the proof of \eqref{eqn:lemma12_3} and \eqref{eqn:lemma12_4} are quicker and similar. Note that 
\begin{align*}
    \eqref{eqn:lemma12_3} \iff \lambda\cdot \frac{e^\lambda - \beta}{e^\lambda - 1} \geq (1-\beta + \beta\lambda) \cdot \frac{e^{1/\lambda}-\beta}{e^{1/\lambda}-1}
\end{align*}
which is equivalent to a polynomial (in $\beta$) of degree $2$ being positive. The leading coefficient of this polynomial $P$ is negative and we notice that $P(1)=0$ and that $P(0)\geq 0$ by \eqref{eqn:lemma12_1b}. All these facts together imply that $P(\beta)\geq 0$ for all $\beta \in [0,1]$. The proof of \eqref{eqn:lemma12_4} is similar. 
\end{proof}{}

\begin{lem}
\label{lem:sequence_lemma}
Let $0<\lambda\leq 1$, $d>0$ and define the following functions ($x\in \mathbb{R}$):
$$ f(x) = \left(1+\frac{1}{d} \right)\cdot x + \frac{1}{d\left((1+1/d)^{\lambda d}-1\right)}$$

$$ g(x) = \left(1+\frac{1}{d} \right)\cdot x + \frac{1}{d\left((1+1/d)^{d/\lambda}-1\right)}$$

Given $S\geq 0$ and a word $w\in \{a,b\}^*$ we define a sequence $S_w$ recursively as follows:

$$S_{w.y} = \left\{
                \begin{array}{lll}
                  S \mbox{ if } w.y=\epsilon\\            
                  f(S_{w})\mbox{ if } y=a\\
                  g(S_{w})\mbox{ if } y=b\\
                \end{array}
              \right.$$
              
Then for any $w\in \{a,b\}^*$ such that $|w|_a + \lambda |w|_b \geq d$ we have that $S_{w}\geq 1$.
\end{lem}
\begin{proof} Let $w'=b\ldots ba\ldots a=b^{|w|_b}a^{|w|_a}$ be the word made of $|w|_b$ consecutive $b$s followed by $|w|_a$ consecutive $a$s. Then we claim that $S_{w'}\leq S_{w}$. This directly follows from the fact that for any real number $x$, $f(g(x))\leq g(f(x))$. Noticing this, we can swap positions between an $a$ followed by a $b$ and reducing the final value. We keep doing this until all the $b$s in $w$ end up in front position.

With standard computations one can check that 
$$S_{b^{|w|_b}} = S\cdot \left(1+1/d\right)^{|w|_b} + \frac{\left(1+1/d\right)^{|w|_b}-1}{(1+1/d)^{d/\lambda}-1} $$

For ease of notation define $S' = S_{b^{|w|_b}}$. Using the assumption that $|w|_a+\lambda |w|_b \geq d$ and that $S\geq 0$ we get that

$$S'\geq \frac{(1+1/d)^{(d-|w|_a)/\lambda}-1}{(1+1/d)^{d/\lambda}-1}$$

Again using standard calculations we get that 

$$S_{w'}\geq S'\cdot (1+1/d)^{|w|_a} + \frac{(1+1/d)^{|w|_a}-1}{(1+1/d)^{\lambda d}-1}$$ 

which implies

$$S_{w'}\geq \frac{(1+1/d)^{(d-|w|_a)/\lambda}-1}{(1+1/d)^{d/\lambda}-1}\cdot (1+1/d)^{|w|_a} + \frac{(1+1/d)^{|w|_a}-1}{(1+1/d)^{\lambda d}-1}$$

Define $h(x)=\frac{(1+1/d)^{(d-x)/\lambda}-1}{(1+1/d)^{d/\lambda}-1}\cdot (1+1/d)^{x} + \frac{(1+1/d)^{x}-1}{(1+1/d)^{\lambda d}-1}$. We finish the proof by proving that for any $0<\lambda \leq 1$, any $d>0$ and any $x\geq 0$, we have that $h(x)\geq 1$.

Note that $h(0)=1$ and that 

$$h'(x) =\ln (1+1/d)\cdot \left(\frac{(1+1/d)^x}{(1+1/d)^{\lambda d}-1}-\frac{1-\lambda}{\lambda}\cdot \frac{ (1+1/d)^{(d-(1-\lambda)x)/\lambda}}{ (1+1/d)^{d/\lambda}-1}-\frac{(1+1/d)^x}{(1+1/d)^{d/\lambda}-1}\right)$$

To study the sign of $h'(x)$ we can drop the $\ln(1+1/d)$ and write

\begin{align*}
    h'(x)\geq 0 \iff& \frac{(1+1/d)^x}{(1+1/d)^{\lambda d}-1}-\frac{1-\lambda}{\lambda}\cdot \frac{ (1+1/d)^{(d-(1-\lambda)x)/\lambda}}{ (1+1/d)^{d/\lambda}-1}-\frac{(1+1/d)^x}{(1+1/d)^{d/\lambda}-1} \geq 0\\
    \iff& \frac{1}{(1+1/d)^{\lambda d}-1}-\frac{1-\lambda}{\lambda}\cdot \frac{ (1+1/d)^{(d-x)/\lambda}}{ (1+1/d)^{d/\lambda}-1}-\frac{1}{(1+1/d)^{d/\lambda}-1} \geq 0
\end{align*}{}
 Clearly the last term is increasing as $x$ increases hence we can limit ourselves to prove that $h'(0)\geq 0$ which we can rewrite
 
 \begin{align*}
     h'(0)\geq 0 \iff& \frac{1}{(1+1/d)^{\lambda d}-1}-\frac{1-\lambda}{\lambda}\cdot \frac{ (1+1/d)^{d/\lambda}}{ (1+1/d)^{d/\lambda}-1}-\frac{1}{(1+1/d)^{d/\lambda}-1} \geq 0\\
     \iff& \frac{1}{(1+1/d)^{\lambda d}-1} \geq \frac{\frac{1-\lambda}{\lambda} \cdot (1+1/d)^{d/\lambda}+1}{ (1+1/d)^{d/\lambda}-1}
 \end{align*}{}
 Which holds by equation (5) of  Lemma \ref{lem:inequalities}.
\end{proof}
}{}

\end{document}
